\documentclass[onefignum,onetabnum]{siamonline250106}
\newcommand{\bs}{\left\{}
\newcommand{\es}{\right.}
\newcommand{\ba}{\begin{array}}
\newcommand{\ea}{\end{array}}
\newcommand{\be}{\begin{equation}}
\newcommand{\ee}{\end{equation}}
 
\usepackage{enumitem}
\usepackage{subcaption}
\usepackage{bbm}
\usepackage{cite} 
\usepackage{amsmath,amssymb,amsfonts}
\usepackage{bbold}
\usepackage{algorithmic}
\usepackage{algorithm}
\usepackage{graphicx}
\usepackage{textcomp} 
\usepackage{xcolor} 
\usepackage{bm}
\usepackage{tikz}

\newsiamremark{remark}{Remark}

\newcommand{\R}[0]{\mathbb{R}}

\newcommand{\bbeta}{\boldsymbol{\beta}}

\newcommand{\balpha}{\boldsymbol{\alpha}}
\newcommand{\bepsilon}{\boldsymbol{\epsilon}}
\newcommand{\orange}{\color{orange}}

\newcommand{\red}{\color{red}}
\newcommand{\green}{\color{green}}
\newcommand{\cH}[0]{\mathcal{H}}

\newcommand{\cS}[0]{{\cal S}}
\newcommand{\sgn}{{\rm sign}}

\date{\today}

\title{Analytical study of the optimal combination of binary classifiers based on classifiers-induced partitioning of the training set}

\author{
Jean-Marc Brossier
\thanks{CNRS, Univ. Grenoble Alpes, Grenoble-INP, GIPSA-lab, Grenoble, France. 
\email{jean-marc.brossier@gipsa-lab.grenoble-inp.fr}}
\and 
Olivier Lafitte
\thanks{Université Sorbonne Paris Nord, LAGA, UMR 7539.
 IRL CNRS-CRM 3457. Université de Montréal. Canada. 
 \email{lafitte@math.univ-paris13.fr}.}}

\headers{Analytical study of the optimal combination of binary classifiers}{J.-M. Brossier, O. Lafitte}

\begin{document} 

\maketitle

\begin{abstract}
  This paper studies an optimal linear combination of binary classifiers based on a logical structuration of the dataset via truth tables. The given classifiers  partition data into equivalence classes, allowing for a rigorous analysis of the convexified empirical risk through a multidimensional generalization of classification calibrated functions. We establish sufficient conditions for the existence and uniqueness of the  (global) point of minimum of the convexified empirical risk for any list of classifiers (when the number of classifiers is large, there frequently could be no point of minimum). In the case of three classifiers, our analysis allows to list all the configurations leading to either a unique solution, infima or non-unique points of minimum. Furthermore, we derive explicit analytical formulae for optimal weights using Exponential (Boost) and Logistic (Logit) loss functions, bypassing iterative optimization. 
  The stability of the resulting classifier and the analysis of data quality can be evaluated through the introduction of the notion of $\phi$-frontiers.
\end{abstract}



\section{Introduction}
\label{sec:introduction}

For supervised classification problems, a way of building a good classifier is to construct a linear combination of less accurate ones that achieves a better accuracy on a given training set.

Many algorithms combine weak classifiers (slightly better than a random guess) to build a strong classifier (that improves accuracy and robustness) by taking advantage of the diversity of weak classifiers. This kind of algorithms is referred to as ensemble Learning.
Two major categories are Bagging and Boosting. 
Bagging algorithms (Bootstrap Aggregating) build multiple instances of the same classifier on bootstrap samples of the dataset and combines them ({\it Bagging Predictors} were introduced by L. Breiman \cite{breiman1996bagging}, {\it random forest} \cite{breiman2001random}  is a famous instance). 
Boosting algorithms combine several weak classifiers sequentially, with each new classifier focusing on the errors of the previous ones; the most well-known and first practical implementation is AdaBoost (Adaptive Boosting) proposed by Freund and Schapire\cite{Freund1995,FreundSchapireExp,FreundSchapire1997}. 
See for example \cite{bach2024learning} (p.303) for a more recent presentation of abstract boosting with a predefined list of classifiers.
Two well-known extensions are Gradient Boosting \cite{Friedman2001} and XGBoost \cite{chen2016xgboost}.

Other types of ensemble learning algorithms include Stacking\cite{wolpert1992stacked} that combines predictions from several heterogeneous models using a meta-classifier for the final prediction and voting\cite{polikar2006ensemble} that combines predictions from different classifiers by majority or weighted voting.

This paper concerns the identification of a resulting classifier $\sgn \sum_1^m \beta_j h_j$ from a linear combination of classifiers in a given list 
\begin{equation}
  \label{eq:Hm}
\cH_m:=\{h_1,\ldots, h_m\}
\end{equation}
known \textit{a priori}.
In our setup, we consider only logical classifiers, that is binary functions $h_j$ from a set of features ${\cal X}$ to $\left\{-1,+1\right\}$.

Optimality of the resulting classifier is understood to mean it minimizes a risk based on the $0/1$ logical loss function as well as fairly general convexifications of it (using classification calibrated loss functions). Note that our study does not construct sequentially the $h_j$ and is thus different from boosting algorithms.

Although the number of misclassified points, the logical $0/1$ risk, is a natural choice, its minimum associated with $\sum_1^m \beta_j h_j$ is achieved at an infinite number of values of $\bbeta:=(\beta_1,\cdots,\beta_m)$ and thus cannot lead to a choice of a unique new classifier (it is an ill-posed problem in the sense of Hadamard). Moreover, in an optimization context, it is neither convex nor continuous, which is problematic from an algorithmic point of view.

That  is where the classification calibrated functions leading to a convexified risk are useful.
In chapter 7 of \cite{SchapireFreundBook2012} and \cite{bartlett06}, different choices of loss functions are used. 

In this paper, we generalize the convexified risk introduced by Bartlett \cite{bartlett06} to a $m$-dimensional convexified risk. 
It is an extension of the elementary three-classifier approach presented in \cite{labro2021,labro2022GRETSI} where the examples were aggregated into a partition of $8$ subsets associated to the correctness or incorrectness of the respective decisions of the three classifiers: each subset contains the examples which return a given answer (true or false) for each classifier, we called this {\bf structuring the truth table}.

This approach relates to the concept of partitioning the input space found in "Mixture of Experts" \cite{jacobs1991adaptive} or "Local Ensembles" strategies (see for example \cite{woods1997combination}), where the behavior of classifiers determines regions of interest. However, instead of training local models, we use this partitioning to aggregate the training set information into a "truth table" of size $2^m$, extracting equivalence classes of indistinguishable examples.

This  partitioning is generalized here to $m$  elementary classifiers, where one defines a truth table whose columns are associated to $2^m$ distinct subsets of indistinguishable (for the given set of $m$ classifiers) examples.
Seen through this prism (see definition \ref{def:tt}), knowledge of the training set reduces to $2^m$ weighted examples: this structuring leads to a compression of the training set; the resulting reduction in complexity comes at the price of the loss of some information.
Note also that this renders tractable the $0/1$ risk by reducing the identification of the minimum of the risk to the minimum of at most $2^m$ values, independently of the number of examples. 

A sort of partitioning according to the same decisions of a single given weak classifier the training set is used in the work of  Schapire \& Singer \cite{schapire1998improved}, but they do not mix different classifiers for the partitioning.

Minimizing the convexified risk amounts to optimize a function with $m$ variables $\beta_j$ and $2^m$ parameters, independently of the number of examples $n$; it depends only on the proportion of examples in each subset. 
It is then  possible to carry out a complete study of the cost function (both analytically and computationally) to determine the best combination of elementary classifiers. In particular, it becomes possible to establish sufficient, or necessary and sufficient conditions for the existence and the uniqueness of a point of minimum, which to the best of our knowledge is new. 

When it is possible to calculate this point of minimum, we obtain directly and analytically the optimal weighting coefficients, i.e. the resulting classifier and the number of misclassified points.

We can then study analytically the stability of the resulting decision via the introduction of $\phi$-frontiers, that is when a small modification of the training set changes the decision. 

Furthermore, we characterize the quality of the data based on the sensitivity of the decision to the choice of the classification calibrated loss function.

As an important byproduct, we show that there exist training sets and classifiers for which there are cases of existence and non uniqueness of the point of minimum and of non existence (the function has an infimum), hence no possibility of constructing an optimal classifier.
Moreover, detecting these cases of infimum is numerically hard: different classical optimization packages (generally designed for $\alpha$-convex functions) can indeed return different "solutions" for the search of a minimum of a function that has an infimum.

Remark that:
(i) Our study is based upon the existence of examples for which contradictions occur between the classifiers chosen, this allows to get a new classifier which classifies correctly more examples. To our knowledge, this remark seems not to have been used before.
(ii) Once the list of classifiers is fixed, we have no way to distinguish two elements of the training set which return the same decisions on all classifiers.

Writing the convexified  risk using the data structuring  we propose is a way  of setting the problem that  opens up many possibilities, including:
\begin{itemize}
  \item identification of  situations in which a unique minimum exists. We have a necessary and sufficient condition (hence an exhaustive inventory of these situations)  in  the case of three classifiers  and   sufficient condition   in  the   general   case, see sections  \ref{sec:ExistUnique}   and \ref{sec:existence-uniqueness-3},
  \item identification of cases of non existence or non uniqueness of a point of minimum. In this case, different regularizations yield a unique point of minimum. However, this limit does not construct a suitable classifier and falls down in a category of $\phi$-frontier points (the data are not of good quality as mentioned above). 
  \item compute analytically the point of minimum of the convexified cost function for different choices of the loss function in  the case of three classifiers,
  \item adding a classifier (section \ref{sec:introductory-example}),  which is straightforward with the proposed structure, and improves the performance of an existing classifier set.
\end{itemize}

\vspace{.5cm}

The paper is organized as follows: \\

Section \ref{sec:structuration} of this paper describes the structuration of the data (the set of examples) according to the given family of classifiers (in a predefined list) and deduce an original form for the empirical risk. 

Section \ref{sec:convexification} describes and uses the convexified risk. In this section, we introduce a multidimensional generalization of the classification calibrated function of Bartlett {\it et al} \cite{bartlett06}.
This generalization introduces a set of parameters $\balpha$ 
in $\R_+^{2^m}$ of sum equal to one.

In section \ref{sec:optimal-classifier}, we give the equations of the unique point of minimum of the convexified risk, when it exists, which is function of $\balpha$ and which leads to the definition of frontiers {\it i.e.} hypersurfaces in the set of parameters $\balpha$ across which the resulting optimal classifier changes. 

Section \ref{sec:labelleprovince} proposes an exhaustive study in the case of three classifiers. 
We list all the cases of infimum or of multiple minima (theorem \ref{theo:zoology-europe}). 

In particular
\begin{itemize}
  \item we get a complete list of cases for which the infimum is zero.
  \item when we have a point of infimum, a regularization yields (of course) a unique point of minimum which corresponds to a $\bbeta^\star$  with a large norme. We will see later that it corresponds to a corner point (an uncertain decision).
\end{itemize}

For the loss functions $x\rightarrow e^{-x}$ (boost) and $x\rightarrow \log_2(1+e^{-x})$ (logit) we calculate the exact
solution of the minimization problem and we use these results in section \ref{sec:zoll} to draw an image of the
frontiers (hypersurfaces in the parameter $\balpha$ for which there is a change of decision when changing $\balpha$) in
a few cases.

All these ideas rely on a way of structuring the data based on contradictions between
classifiers, which is the aim of section \ref{sec:structuration}.

\section{Structuration of the data}
\label{sec:structuration}

In order to do this, we introduce the general notations for the supervised learning problem for two classes. 

\subsection{Notations and setup of the problem}
\label{sec:notations}
We deal with supervised learning with ${\cal X}={\R}^d$ for space characteristics associated with two labels $\left\{-1,+1\right\}$, attached to two classes.

The goal is to classify an object (decide if its label is $\pm 1$) from $x\in{\cal X}$. 
A supervised algorithm uses
$\cS =\left\{(x_i,y_i), i=1\ldots n \right\} \subset {\cal X}\times \left\{-1,+1\right\} $ to estimate a classifier $h : {\cal X} \rightarrow \left\{-1,+1\right\}$.

The empirical risk of the classifier $h$ on $\cS $ writes $\frac 1n \sum_{(x_i,y_i)\in \cS}  \mathbb{1}_{y_ih(x_i)<0}$. 
From now on, we consider a family of classifiers $\cH_m:=\{h_1,\ldots, h_m\}$.  The empirical logical ($0/1$) for the classifier $\sgn\sum \beta_j h_j$ is thus the function from $\R^m$ to $\R$ given by
\begin{equation}
  \label{eq:logical-risk}
  {\cal R}^{\cS,{\cal H}_m}_{\mathbb{1}_{\bullet<0}}(\bbeta)  :=
  \frac 1n \sum_{(x_i,y_i)\in \cS}  \mathbb{1}_{y_i(\sum \beta_j h_j)(x_i)<0}.
\end{equation}

\begin{definition}[Truth table]
  \label{def:tt}
  The truth table of a family of classifiers $\cH_m:=\{h_1,\ldots, h_m\}$ is the partition of $\cS $ in $2^m$ classes associated to $\cH_m$, where each set of the partition groups the examples which return the same decision for all classifiers, the number of such examples is called an entry of the truth table. These entries can be  replaced by proportions.
\end{definition}

Note that, throughout this paper, we call \textit{decision} the correctness with respect to the expected result and not the actual value predicted by the classifier: 
for a classifier $h_j$, one defines the decision $G_j(x_i,y_i)=y_i h_j(x_i)$, therefore we note $+1$ the correct decisions, such as $G_j(x_i,y_i)=+1$.

Two cases can occur which leads to interesting results:

\begin{definition}
  \label{def:equivalent-perfect}
  \mbox{}
  \begin{enumerate}
    \item A classifier is perfect if it returns the same decision for all examples. 
    \item   Two classifiers are equivalent if they return the same decision for each example. 
  \end{enumerate}
\end{definition}

Remark that in all this study, replacing a classifier by its negation does not change anything; in particular, in definition \ref{def:equivalent-perfect} we include the case of a classifier returning the wrong decision for all examples and we also include the case of two classifiers contradictory on all examples. 

\subsection{Introductory example with one to three classifiers}
\label{sec:introductory-example}

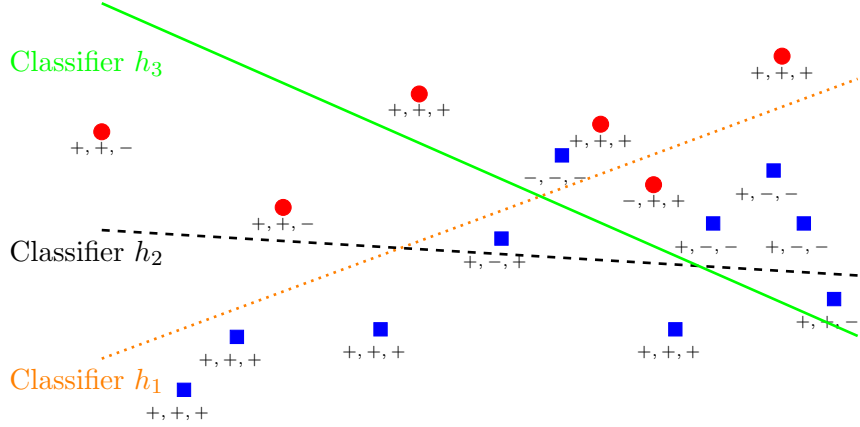
\begin{figure}[H]
  \begin{center} 
    \begin{tikzpicture}[scale=1]
    \filldraw [red] (-5.0,1.5) circle (3pt); \draw(-5.0,1.5) node[below]{{\tiny $+,+,-$}} ;
    \filldraw [blue] (-4.,-2) rectangle ++(5pt,5pt); \draw(-4.,-2) node[below]{{\tiny $+,+,+$}} ;
    \filldraw [blue] (-3.3,-1.3) rectangle ++(5pt,5pt); \draw(-3.3,-1.3) node[below]{{\tiny $+,+,+$}} ;
    \filldraw [red] (-2.6,.5) circle (3pt); \draw(-2.6,.5) node[below]{{\tiny $+,+,-$}} ;
    \filldraw [blue] (-1.4,-1.2) rectangle ++(5pt,5pt); \draw(-1.4,-1.2) node[below]{{\tiny $+,+,+$}} ;
    \filldraw [red] (-0.8,2.) circle (3pt); \draw(-0.8,2.) node[below]{{\tiny $+,+,+$}} ;
    \filldraw [blue] (+0.2,0.) rectangle ++(5pt,5pt); \draw(+0.2,0.) node[below]{{\tiny $+,-,+$}} ;
    \filldraw [blue] (+1.0,1.1) rectangle ++(5pt,5pt); \draw (+1.0,1.1) node[below]{{\tiny $-,-,-$}} ;
    \filldraw [red] (+1.6,1.6) circle (3pt); \draw (+1.6,1.6) node[below]{{\tiny $+,+,+$}} ;
    \filldraw [red] (+2.3,0.8) circle (3pt); \draw (+2.3,0.8) node[below]{{\tiny $-,+,+$}} ;
    \filldraw [blue](+2.5,-1.2) rectangle ++(5pt,5pt); \draw (+2.5,-1.2) node[below]{{\tiny $+,+,+$}} ;
    \filldraw [blue] (+3,0.2) rectangle ++(5pt,5pt); \draw (+3,0.2) node[below]{{\tiny $+,-,-$}} ;
    \filldraw [blue](+3.8,0.9) rectangle ++(5pt,5pt); \draw (+3.8,0.9) node[below]{{\tiny $+,-,-$}} ;
    \filldraw [red] (+4.0,2.5) circle (3pt); \draw (+4.0,2.5) node[below]{{\tiny $+,+,+$}} ;
    \filldraw [blue](+4.2,0.2) rectangle ++(5pt,5pt); \draw (+4.2,0.2) node[below]{{\tiny $+,-,-$}} ;
    \filldraw [blue](+4.6,-.8) rectangle ++(5pt,5pt); \draw (+4.6,-.8) node[below]{{\tiny $+,+,-$}} ;
    \draw[line width=1,orange, dotted] (-5,-1.5)--(+5,+2.2) ;
    \draw[line width=1,black, dashed] (-5,0.2)--(+5,-.4) ; 
    \draw[line width=1,green] (-5,3.2)--(+5,-1.2) ;
    \draw[color=orange] (-5.2,-1.5) node[below]{Classifier $h_1$} ;
    \draw[color=black] (-5.2,0.2) node[below]{Classifier $h_2$} ;
    \draw[color=green] (-5.2,2.7) node[below]{Classifier $h_3$} ;
    \end{tikzpicture}
  \end{center}  
  \caption{\label{fig:ex-3-with-table}$16$ examples and $3$ classifiers which yields the truth table \ref{table:ex-3-classifers}. The label $y_i$ of blue squares is $-1$, the label of red circles is $+1$. Labels such as $+,-,+$ mean $G_1=+1,G_2=-1,G_3=+1$: $h_1$ is true, $h_2$ is false, $h_3$ is true.}
\end{figure}

We construct truth tables associated with $\cS $ and with a set of three classifiers shown in figure \ref{fig:ex-3-with-table}.

\paragraph*{One classifier}
For $1$ classifier $\{h_1\}$: $h_1$ separates all but $2$ examples (figure \ref{fig:ex-3-with-table}).
For this given set of $16$ examples $\cS $ and classifier $h_1$: $\sharp \{G_1(x_i,y_i)=-1\}=2$  and $\sharp \{G_1(x_i,y_i)=+1\}=14$, $n=16=2+14$, where $\sharp A$ denotes the number of elements of the set $A$. 
This begins a binary classification tree associated with a truth table with two columns. Note that it corresponds to the root of the tree in figure \ref {fig:binary-tree-3}.  

\paragraph*{Two classifiers}
For $2$ classifiers $\{h_1,h_2\}$, each column of the truth table is split into two columns hence leading to a truth table with four columns whose sum of elements is $n$. 

\paragraph*{Three classifiers}
For $3$ classifiers $\{h_1,h_2,h_3\}$, each column is associated with the set of points in $\cS $ that match a given configuration $\bepsilon=(\epsilon_1,\epsilon_2,\epsilon_3)\in \{-1,1\}^3$ of the three classifiers, the column header provides the cardinal of this set: $\sharp \{i: G_1(x_i,y_i)= \epsilon_1, G_2(x_i,y_i)= \epsilon_2, G_3(x_i,y_i)= \epsilon_3 \}$.
This corresponds to the leaves of the tree on figure \ref{fig:binary-tree-3}.

The construction of all the columns is shown in the binary tree in figure \ref{fig:binary-tree-3}. The resulting truth table is shown in table \ref{table:ex-3-classifers}.  
The structure of the tables are independent of the order in the list of classifiers. 

\begin{figure}[H]
\begin{center} 
  \begin{tikzpicture}[scale=.8]

    \draw (0,0) node[below]{
      \begin{small}\fbox{$
      \begin{array}{rrr}
                                            & {\red 2} & {\red 14} \\
        {\orange G_1} \;  \hspace{-.35cm}   & -1 & +1              \\
      \end{array} 
      $}\end{small}} ;

      \draw[line width=.7,->] (0,-0.16)--(0,+.1) ;
      \draw[line width=.7,->] (0,-0.15)--(4.6,+.15) ;
      \draw (0,+2.) node[below]{
        \begin{small} \fbox{$\begin{array}{rrr}
                                             & {\color{red}1} & {\color{red}1} \\
          {\orange G_1} \;  \hspace{-.35cm}  & -1             & -1 \\
          {G_2}  \;\hspace{-.35cm}           & -1             & +1  \\
        \end{array} 
        $}\end{small}} ;

      \draw (+6,+2.) node[below]{
        \begin{small}\fbox{$\begin{array}{rrr}
                                          & {\red 4} & {\red 10} \\
        {\orange G_1} \;  \hspace{-.35cm} & +1       & +1 \\
        {G_2}  \; \hspace{-.35cm}         & -1       & +1 \\
        \end{array} $}\end{small}} ;

      \draw[line width=.7,->] (0,1.83)--(0,+2.12) ;
      \draw (0,+4.5) node[below]{
        \begin{small}\fbox{$\begin{array}{rrr}
                                                  & {\red 1} & {\red 0} \\
          {\orange G_1} \;  \hspace{-.35cm}       & -1       & -1   \\
          {G_2}  \; \hspace{-.35cm}               & -1  & -1  \\
          {\green G_3}  \; \hspace{-.35cm}        & -1  & +1  
        \end{array} $}\end{small}} ;

        \draw[line width=.7,->] (0,1.84)--(1.6,+2.11) ;
        \draw (3.,+4.5) node[below]{
          \begin{small}\fbox{$\begin{array}{rrr}
                                              & {\red 0} & {\color{red}1} \\
            {\orange G_1} \;  \hspace{-.35cm} & -1  & -1   \\
            {G_2}  \; \hspace{-.35cm}  & +1   & +1   \\
            {\green G_3}  \; \hspace{-.35cm}  & -1  & +1  
          \end{array} $}\end{small}} ;

          \draw[line width=.7,->] (6,1.83)--(6,+2.11) ;
          \draw (6.,+4.5) node[below]{
            \begin{small}\fbox{$\begin{array}{rrr}
                                               & {\red 3} & {\red 1} \\
              {\orange G_1} \; \hspace{-.35cm} & +1  & +1            \\
              {G_2}  \; \hspace{-.35cm}  & -1  & -1                  \\
              {\green G_3}  \; \hspace{-.35cm} & -1  & +1  
            \end{array} $}\end{small}} ;

          \draw[line width=.7,->] (6,1.84)--(7.6,+2.11) ;
          \draw (9.,+4.5) node[below]{
            \begin{small}\fbox{$\begin{array}{rrr}
                                               & {\red 3} & {\color{red}7} \\
              {\orange G_1} \; \hspace{-.35cm} & +1       & +1             \\
              {G_2}  \; \hspace{-.35cm}        & +1       & +1             \\
              {\green G_3}  \; \hspace{-.35cm} & -1       & +1
            \end{array} $}\end{small}} ;
  \end{tikzpicture}
\end{center}
\caption{\label{fig:binary-tree-3} Binary tree and values for the three classifiers of  figure \ref{fig:ex-3-with-table}.}
\end{figure}
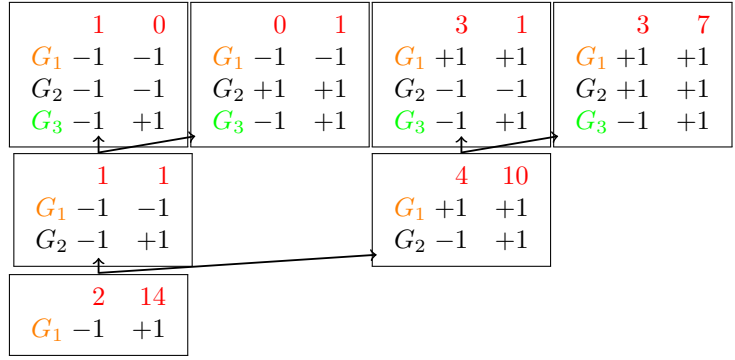

In the next section, we extend this structuration of the data to $m$ classifiers,  which yields a new insight on the examples of $\cS$.

\subsection{The truth table as a partition of the dataset}
\label{sec:classification-tree}

For $m$ classifiers, for each example $(x_i,y_i)$, the classifier $h_j$ can be true ($\epsilon_j=G_j(x_i,y_i)=+1$) or
false ($\epsilon_j=-1$).  For each answer of the type (TFFTF...FT), we can group all the examples that return this
answer. This defines the partition of $\cS$ of definition \ref{def:tt}.

Two labelings are natural (table or tree): the first one relies on the global list $\{h_1,\ldots, h_m\}$, and the second one is
associated with a tree and relies on the ordered list $\{h_1\}$, $\{h_1,h_2\}$, ..., $\{h_1,\ldots, h_m\}$. The mapping
between the two labelings is straightforward (see \eqref{eq:labelings}).

The two different labelings we propose for the truth table are as follows:

\paragraph{Labeling of the truth table for a list of classifiers}
The answer of the $m$ classifiers is given by  $\bepsilon=(\epsilon_1,\ldots, \epsilon_m) \in \{-1,1\}^m$. To each $\bepsilon$ one associates 
$$
k(\bepsilon):=\sum_{j=1}^m 2^{m-j}\frac{1+\epsilon_j}{2}.
$$
This transformation is a bijection from $\{-1,+1\}^m$ to $\{0,2^m-1\}$, whose reciprocal is denoted by $k\rightarrow \bepsilon(k)$.

The truth table associated with $(\cS, \cH_m)$ is characterized by 
$$
\balpha^{\cS,\cH_m} = (\alpha_0^{\cS, \cH_m},\ldots, \alpha_{2^m-1}^{\cS, \cH_m})
$$ 
where
$$
\alpha^{\cS,\cH_m}_{k(\bepsilon)}=\alpha^{\cS, \cH_m}(\bepsilon)=\frac{\sharp\{i: y_ih_j(x_i)=\epsilon_j, \forall j \}}{\sharp \cS }
$$
is the proportion ($\alpha_k^{\cS, \cH_m}\in \mathbb{Q}_+$ and $\sum_{k=0}^{2^m-1} \alpha_k^{\cS, \cH_m}=1$) of examples of $\cS $ in the logical configuration $\bepsilon$.
          
Introduce the set of real positive parameters which is used in the sequel 
\begin{equation}
  {\cal A}^m = \R_+^{2^m}\cap \{\sum_{k=0}^{2^m-1} \alpha_k=1\}.
  \label{eq:domain}
\end{equation}

The configuration $\balpha^{\cS,\cH_m}$ belongs to ${\cal A}^m$.
  
Note that in this section and in this labeling, $\alpha_1$ (for example) has different expressions depending on the choice of $m$.
While we will mostly use the present notations in this paper, we introduce in the next paragraph  a labeling which overcomes this notational problem.

\paragraph{Tree structure labeling of the truth table}
When $m$ is given in advance, the labeling proposed in the previous subsection is natural.

If one wants to refine the analysis by increasing the number of classifiers, an incremental labeling introduced in \cite{kempen} for
genealogical purposes in the $16$th century, called later the Sosa-Stradonitz numbering, is a very good choice which allows to label, independently of the number of classifiers, each class of the classification tree.  
This alternative labeling is illustrated in figure \ref{fig:alt-label} (for example $c_4=c_8+c_9, c_6=c_{12}+c_{13},\cdots$). 
 
\begin{figure}[H]
\begin{center}
\begin{tikzpicture}[
  grow = down,level distance = 3em, edge from parent/.style = {draw, -latex},
  every node/.style = {font=\scriptsize},
  level/.style={sibling distance=54mm/#1,dot/.default = 10pt,minimum size=7pt},
  sloped
  ]
  \tikzset{vertex/.style={circle, inner sep=2pt, outer sep=0pt, minimum width=.6cm}}
  \node[vertex] [rectangle,draw] (z){$c_1=16$}
    child {node[vertex] [rectangle,draw]  {$c_2=14$}
      child {node[vertex] [rectangle,draw] {$c_4=10$}
        child {node[vertex] [rectangle,draw] 
        {$\begin{array}{c}c_{8}\\ \shortparallel\\ 7\end{array}$}}
        child {node[vertex] [rectangle,draw] 
        {$\begin{array}{c}c_{9}\\ \shortparallel\\ 3\end{array}$}}}
      child {node[vertex] [rectangle,draw] (e) {$c_5=4$}
        child {node[vertex] [rectangle,draw] 
        {$\begin{array}{c}c_{10}\\ \shortparallel\\ 1\end{array}$}}
        child {node[vertex] [rectangle,draw] (g) 
        {$\begin{array}{c}c_{11}\\ \shortparallel\\ 3\end{array}$}}}}
    child {node[vertex] [rectangle,draw] (h) {$c_3=2$}
      child {node[vertex] [rectangle,draw] (i) {$c_6=1$}
        child {node[vertex] [rectangle,draw] (j) 
        {$\begin{array}{c}c_{12}\\ \shortparallel\\ 1\end{array}$}}
        child {node[vertex] [rectangle,draw] (k) 
        {$\begin{array}{c}c_{13}\\ \shortparallel\\ 0\end{array}$}}}
      child {node[vertex] [rectangle,draw] (l) {$c_{7}=1$}
        child {node[vertex] [rectangle,draw] (m) 
        {$\begin{array}{c}c_{14}\\ \shortparallel\\ 0\end{array}$}}
        child {node[vertex] [rectangle,draw] (n) 
        {$\begin{array}{c}c_{15}\\ \shortparallel\\ 1\end{array}$} 
        }}
  };  
\end{tikzpicture}
\end{center}
\caption{Alternative labeling of the classification tree}
\label{fig:alt-label}
\end{figure}
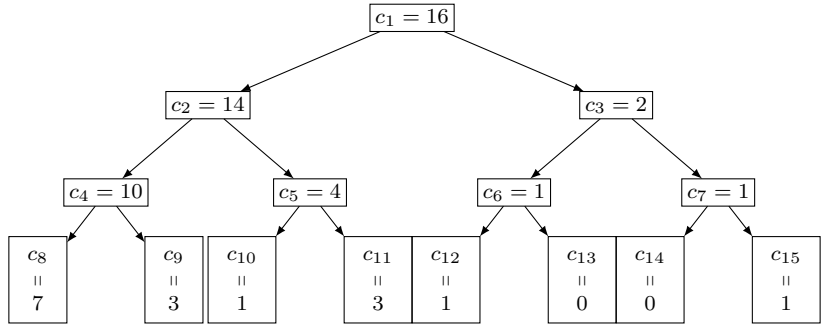

\begin{table}[H]
  $$ 
\begin{array}{ccccccccc}
  \hline
  & \alpha_0 & \alpha_7 & \alpha_4 & \alpha_3 & \alpha_2 & \alpha_5 & \alpha_1 & \alpha_6 \\
  & c_8      &c_{15}    & c_{12}   & c_{11}   & c_{10}   & c_{13}   & c_9      & c_{14} \\
  \hline 
  {\color{orange} G_1} \;  \hspace{-.35cm} & -1  & +1  & +1  & -1  & -1  & +1  & -1 & +1 \\ 
  {\color{black} G_2}  \; \hspace{-.35cm}  & -1  & +1  & -1  & +1  & +1  & -1  & -1 & +1 \\
  {\color{green} G_3}  \; \hspace{-.35cm}  & -1  & +1  & -1  & +1  & -1  & +1  & +1 & -1 \\
  \hline
  \mbox{Cardinal} & 1        & 7        & 3        & 1        & 0        & 1        & 0        & 3    \\
\end{array} 
$$ 
\caption{\label{table:ex-3-classifers}Truth table for the partition associated to the $m=3$ classifiers and $n=16$ examples given in figure \ref{fig:ex-3-with-table}. Note that $n=16= \sum_{i=8}^{15} c_i$ and $\alpha_i = c_{8+i}/n$.
}
\end{table}

\begin{proposition}
  \label{prop:properties}
  We have the following properties:

  \noindent (i) The truth table is independent of the order of the list $\cH_m:=\{h_1,\ldots, h_m\}$.

  \noindent (ii) When two classifiers $h_j,h_k$ return the same answer for all examples, the partition generated by $\cH_m$ has
  $2^{m-1}$ empty columns and the other columns return the partition generated by $\cH_m \setminus h_k$. 
  This is also the case when $h_k$ is the negation of $h_j$. 
\end{proposition}

\begin{proof}\mbox{}

 \noindent (i)  for a permutation $\sigma$ of the classifiers $h_1,\cdots,h_m$, there exists a matrix $M(\sigma)$ of size $2^{m+1}\times 2^{m+1}$ such that 
  $(c_1^{\sigma(\cH_m)},\cdots,c_{2^{m+1}}^{\sigma(\cH_m)})^\top = M(\sigma) (c_1^{\cH_m},\cdots,c_{2^{m+1}}^{\cH_m})^\top $.
  This does not change the structure of the problem. 

 \noindent (ii)
  We can identify the cases where we have two equivalent classifiers or two contradictory classifiers: we check that the classifiers $h_p, h_q$, $p<q$ are equivalent in the list $\{h_1,\ldots, h_m\}$ with respect to the set of examples $\cS $ if and only if, for all $k$ such that $\epsilon_p(k)+\epsilon_q(k)=0$, $\alpha_k=0$ and two classifiers are contradictory if and only if, for all $k$ such that $\epsilon_p(k)=\epsilon_q(k)$, $\alpha_k=0$. 
For example $h_1$ and $h_3$ are equivalent if and only if $c_9=c_{11}=c_{12}=c_{14}=0$, $h_2$ and $h_3$ are contradictory if and only if $c_8=c_{11}=c_{12}=c_{15}=0$. 
Note that, when there is a pair of classifiers which provide the same logical answer or the opposite one, we can suppress one of the two classifiers.
\end{proof}

We can also perform the two following elementary operations:

{\bf Adding a classifier} $h_{m+1}$ to the list  $\{h_1,\ldots, h_m\}$ (easy use of the Sosa-Stradonitz labeling). 
  It corresponds to a splitting of each column of the truth table into two new columns. 
  Let $\bepsilon \in \{-1,1\}^{m+1}$. 
  For $\bepsilon'\in \{-1,1\}^m$, one has indeed
  
  $$n \balpha^{S, \cH_{m+1}}((\bepsilon’, 1)) 
  =\sharp\{i, (y_i h_j(x_i))_{j=1,\ldots, m}=\bepsilon’, y_ih_{m+1}(x_i)=1\}$$
  and 
  $$n \balpha^{S, \cH_{m+1}}((\bepsilon’, -1)) =\sharp\{i, (y_i h_j(x_i))_{j=1,\ldots, m}=\bepsilon’, y_ih_{m+1}(x_i)=-1\}.$$
  
  Notice that 
$\balpha^{S, \cH_m}(\bepsilon’)= \balpha^{S, \cH_{m+1}}((\bepsilon’, -1))+ \balpha^{S, \cH_{m+1}}((\bepsilon’, 1))$
  where $(\bepsilon’,\pm 1)$ describes $\{-1,1\}^{m+1}$.
  
  Note that in the Sosa-Stradonitz notations for the nodes of the tree, this corresponds to splitting
  \begin{equation}
    \label{eq:labelings}
 n \alpha^{\cS,\cH_m}_{k(\bepsilon)} = c^{\cS,\cH_{m}}_{k(\bepsilon)} =   c^{\cS,\cH_{m+1}}_{2k(\bepsilon)} + c^{\cS,\cH_{m+1}}_{2k(\bepsilon)+1}.
  \end{equation}

{\bf Removing a classifier} from the list. We put the classifier we want to remove at the end of the list through the permutation $\sigma_j$ of $\cH_m$ exchanging $h_j$ and $h_m$, and we obtain $c_l^{\cH_m \setminus h_j}= c_{2l}^{\sigma_j(\cH_m)}+c_{2l+1}^{\sigma_j(\cH_m)}$ for all $l\in \{1,\cdots,2^m-1\}$.\\

\subsection{Rewriting the  logical risk}
\label{sec:rewriting-the-risk}

For each $\bbeta\in\R^m$ one introduces, for $\bepsilon \in \{-1,+1\}^m$, thus for
$k(\bepsilon)\in \{ 0, \cdots, 2^{m-1}-1 \}$
\begin{equation}
  \begin{array}{lll}
    X_{k(\bepsilon)}(\bbeta)   &=& \sum_{j=1}^m \epsilon_j \beta_j = \bepsilon \cdot \bbeta \\
    X_{2^m-1-k(\bepsilon)}(\bbeta) &=& - X_{k(\bepsilon)},
  \end{array}
  \label{eq:iiXk}
\end{equation}

which leads to the definition of pairs of decisions:

\begin{definition}[pair]
  \label{def:pair}
  Let $\balpha=(\alpha_0,\ldots, \alpha_{2^m-1})\in {\cal A}^m$.  
   A pair with label $k_l \in 0 \ldots 2^{m-1}-1$ is the couple 
$$(\alpha_{k_l}, \alpha_{2^m -1 - k_l}) := 
  \left( \alpha(\bepsilon(k_l)),\alpha(-\bepsilon(k_l))\right),$$ 
  thanks to
$k(\bepsilon)+k(-\bepsilon)=\sum_{j=1}^m 2^{m-j}(\frac{1+\epsilon_j}{2}+\frac{1-\epsilon_j}{2})= 2^{m}-1$, $\forall \bepsilon\in \{-1,1\}^m$.
\end{definition}

This defines $X_k(\bbeta)$ for all $k \in \{ 0, \cdots, 2^{m}-1 \}$.
For the sake of simplicity of notations, we will write $X_k$ instead of $X_k(\bbeta)$ throughout the paper.
For example
$X_0 = -\beta_1 - \beta_2,
X_1 = -\beta_1 +\beta_2, ...$, 
of figure \ref{fig:sectors} follow this definition for $m=2$.\\

As ${\cal R}^{\cS,{\cal H}_m}_{\mathbb{1}_{\bullet<0}}(\bbeta) =  \min {\cal R}^{\cS,{\cal H}_m}_{\mathbb{1}_{\bullet<0}}$  is a union of some cones (for example, two different sectors in $\R^2$ or different infinite tetrahedrons in $\R^3$), there is no criterion to choose a precise value of $\bbeta$ and, since different values of $\bbeta$ may lead to the minimum value of the empirical risk, these different solutions  could provide different answers on a new example.

As examples returning the same answers for all the classifiers of ${\cal H}_m$ cannot be distinguished, the complexity depends only on $m$, which allows to rewrite the empirical risk as a sum on the partition of size $2^m$ given by the truth table. {\bf In particular, the complexity does not grow when the number $n$ of examples grows}.

Since the same value of $\bepsilon$ is associated with all $(x_i,y_i)$ in one column of the truth table, 
${\cal R}^{\cS,{\cal H}_m}_{\mathbb{1}}(\bbeta)$ can be rewritten on the columns of the truth table, each being characterized by $G_j(x_i,y_i)=\epsilon_j,\forall j$:
$$
{\cal R}^{\cS,{\cal H}_m}_{\mathbb{1}_{\bullet<0}}(\bbeta) =
\frac 1n \sum_{(x_i,y_i)\in \cS}   
\mathbb{1}_{(\sum \beta_j G_j(x_i,y_i))<0} = 
\sum_{\bepsilon \in \{-1,1\}^m}
\sum_{i:G_j(x_i,y_i)=\epsilon_j,\forall j} 
\frac 1n \mathbb{1}_{\bepsilon\cdot\bbeta<0}.
$$
Counting all the examples in each column of the truth table, one has:
$$
\sum_{\bepsilon \in \{-1,1\}^m} \left(
\sum_{i:G_j(x_i,y_i)=\epsilon_j,\forall j} 
\frac 1n \right) \mathbb{1}_{\bepsilon\cdot\bbeta<0} = 
\sum_{\bepsilon \in \{-1,1\}^m} 
\frac{\sharp \{ i:G_j(x_i,y_i)=\epsilon_j,\forall j\}}{n}
\mathbb{1}_{\bepsilon\cdot\bbeta<0}.
$$
The function of $\bbeta$ is then characterized by the signs of all $X_k$:
$$ {\cal R}^{\cS,{\cal H}_m}_{\mathbb{1}_{\bullet<0}}(\bbeta) = 
\sum_{\bepsilon \in \{-1,1\}^m}
\alpha^{\cS, \cH_m}(\bepsilon) \mathbb{1}_{\bepsilon\cdot \bbeta<0} = 
\sum_{k=0}^{2^{m}-1}  \alpha^{\cS, \cH_m}_k \mathbb{1}_{X_k<0}.
$$ 

As $\bepsilon(2^m-k)=-\bepsilon(k)$, we pair $\alpha^{\cS, \cH_m}(\bepsilon(k))$ and $\alpha^{\cS, \cH_m}(\bepsilon(2^m-k))$ associating $X_k = X(\bepsilon(k))$ (and $-X_k= X(-\bepsilon(k))), k\in 0\cdots 2^{m-1}-1$ which yields, for $\bbeta \in \R^m \setminus \cup \{X_k=0\}$:

$$ 
{\cal R}^{\cS,{\cal H}_m}_{\mathbb{1}_{\bullet<0}}(\bbeta) =
\sum_{k=0}^{2^{m-1}-1} 
\alpha_k^{\cS, \cH_m} \mathbb{1}_{X_k<0} 
+ \alpha_{2^m-1-k}^{\cS, \cH_m} \mathbb{1}_{X_k \geq 0}.
$$
This is the structuration of the data that we deal with from now on. 

Note that in this structuration, we deal with two different sets of signs: 
the signs of $y_i h_j(x_i)=G_j(x_i,y_i)$ in $\{-1,+1\}^m$ which relate to the set of examples and decisions and the signs of $X_k, k=0\cdots 2^{m-1}-1$ which separate regions of $\R^m$.

The classical method is to replace $\mathbb{1}_{\bullet < 0}$ by $\phi(\bullet)$ to convexify the logical risk in order to use the powerfulness of convex optimization tools.

The convexified risk is denoted in the sequel 
\begin{equation}
{\cal R}^{\cS,{\cal H}_m}_\phi(\bbeta) =
\sum_{k=0}^{2^{m-1}-1} 
\alpha_k^{\cS, \cH_m} \phi(X_k) 
+ \alpha_{2^m-1-k}^{\cS, \cH_m} \phi(-X_k).
\label{mBartlettSH}
\end{equation}

Note that, for $m=1$, 
\begin{equation}
{\cal R}^{\cS,{\cal H}_1}_\phi(\bbeta) =
\alpha_0^{\cS, \cH_1} \phi(\beta_1) 
+ \alpha_1^{\cS, \cH_1} \phi(-\beta_1) = 
\alpha_0^{\cS, \cH_1} \phi(\beta_1) 
+ (1-\alpha_0^{\cS, \cH_1}) \phi(-\beta_1).
\end{equation}

\section{General classification calibrated functions}
\label{sec:convexification}

The main tool used to study this convexified logical risk is the natural generalization of an object introduced by Bartlett {\it et al} \cite{bartlett06}. This abstract object does not refer to any dataset.

\subsection{The convexified $\phi$-risk of Bartlett}

 This object is the generic conditional $\phi$-risk\footnote{The corresponding notation of Bartlett is $C_\eta(\alpha) = \eta \phi(\alpha)+(1-\eta) \phi(-\alpha)$.} 

\begin{equation}
  \label{eq:GCPhiRisk}
  {\bf C}_\phi^{(\alpha_0,1-\alpha_0)} (\beta) = \alpha_0 \phi(\beta) + (1-\alpha_0) \phi(-\beta).
\end{equation}

They deduced what they called optimal conditional $\phi$-risk 
$$
H(\alpha_0)=\inf_\mathbb{R} {\bf C}_\phi^{(\alpha_0,1-\alpha_0)} (\beta),
$$ 
and Bartlett {\it al} \cite{bartlett06} (p.141) introduced a function $\beta^\star(\alpha_0)$ from $(0,1)$ to $\R$ which is the $\arg \min$ of $H$. 

They restricted their study to loss functions $\phi$ that are {\it classification calibrated} which is equivalent, in the case $\phi$ convex, to $\phi$ differentiable at $0$ and $\phi'(0)<0$. They used such functions to deduce results on the control of the generalized $\phi$-risk by the empirical $\phi$-risk (\cite{bartlett06}, theorem 4, p.148). 

Note that the conditional $\phi$-risk has a unique point of minimum for $\alpha_0(1-\alpha_0)\neq 0$ hence $H$ is well defined for $\alpha_0(1-\alpha_0)\neq 0$.

Note that when $\alpha_0\in\{0,1\}$, $H(\alpha_0)=0$, but it is the infimum of ${\bf C}_\phi^{(0,1)}$ or ${\bf C}_\phi^{(1,0)}$.

In this study we use functions $\phi$ more regular than the classification calibrated functions introduced in section
2.2 of \cite{bartlett06} and we generalize the notion of conditional $\phi$-risk of \cite{bartlett06} to the multidimensional
case for loss functions which are classsification calibrated and at least $C^1$.

\begin{definition}[regularly classification calibrated function]
  A function $\phi$ is said to be regularly classification calibrated if, 
  in addition to being classification calibrated, it is of class $C^1$, 
  strictly convex, strictly decreasing ($\phi'<0$), 
  and of even part infinite at infinity.
\label{def-phi} 
\end{definition}

Note that this function satisfies $\phi'\rightarrow l$ at $-\infty$, where $l$ is strictly negative (possibly infinite).

Three usual choices for $\phi$ are depicted on figure \ref{fig:classic-phi}. 
A common choice in neural networks is $x \rightarrow \phi(x)=\mbox{\small hinge}(x)= \max (0,1-x)$.
However, even if its local derivative is very simple, this function does not have the required properties: it is neither $C^1$ nor strictly convex. The two other functions of figure \ref{fig:classic-phi}, $x \rightarrow \exp(-x)$ and $x \rightarrow \log_2(1+e^{-x})$, are regular $C^1$ and strictly convex and yield analytic results (see section \ref{sec:labelleprovince}).
None of these functions has a minimum however their even part is infinite at infinity and thus has a minimum (see \cite{bartlett06}).
\begin{figure}[H]
  \begin{center}
          \begin{tikzpicture}[scale=.7]
            \draw[color=black] (-3,0)--(0,0) ;
            \draw (.2,1) node[right] {$1$} ; 
            \draw (0,0) node[below] {0} ;

            \draw[color=red] (0,0)--(3,0);\draw[color=red] (-3,1) -- (0,1);\draw[color=red] (0,0) -- (0,1) ;
            \draw[color=red] (-1.5,1.7) node[below]{$\mathbbm{1}_{x<0}$} ;

            \draw[color=blue] (-3,4)--(1,0);\draw[color=blue] (.98,0.02) -- (3,0.02) ;
            \draw[color=blue] (-2.5,2.7) node[below]{$\mbox{\small hinge}(x)$} ;

            \draw[color=brown,dotted]plot [color=red,domain=-1:3] ( \x, {exp(-\x)} );
            \draw[color=brown] (0,3.) node[below]{$\exp(-x)$} ;
            \draw[color=green,dashed]plot [domain=-3:3] ( \x, {ln(1+exp(-\x))/ln(2)} );
            \draw[color=green] (0,3.7) node[below]{$\log_2(1+\exp(-x))$} ;
          \end{tikzpicture}
  \caption{Classical classification calibrated functions $\phi$\label{fig:classic-phi}. Note that $\mbox{\small hinge}(x)\sim_{\pm\infty} \ln(2) \log_2(1+\exp(-x))$ }
\end{center}
  \end{figure}
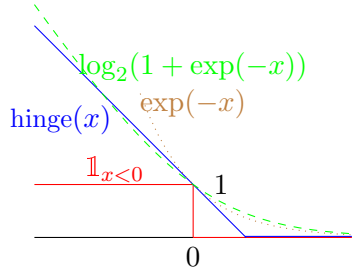

\subsection{Generalization}
\label{sec:generalization-of-classification-calibrated-functions}

One generalizes ${\bf C}_\phi^{(\alpha_0,1-\alpha_0)} (\beta) $ to a function of $m$ variables $\bbeta\in\R^m$ and $2^m$ parameters $\balpha\in {\cal A}^m$.

\begin{definition}[generic multidimensional $\phi$-risk]
Let $\phi$ be regularly classification calibrated (definition \ref{def-phi}). A natural generalization to $m$ variables of the generic conditional $\phi-$risk of Bartlett {\it et al} \cite{bartlett06} is the function of $m$ variables 
\begin{equation}
  \begin{array}{lll}
    {\bf C}_\phi^{\balpha}(\bbeta)
    &=& \sum_{k=0}^{2^m-1}\alpha_k \phi(X_k) \\
    &=&   \sum_{k=0}^{2^{m-1}-1}\alpha_k \phi(X_k)+\alpha_{2^m-1-k} \phi(-X_k)
  \end{array}
  \label{mBartlett}
  \end{equation}
with $\sum\alpha_k=1$.
It is called in the sequel the generic multidimensional $\phi$-risk, of parameter $\balpha$.
\label{def-phi-m}
\end{definition}  

Remark that one may link the multidimensional $\phi$-risk with the conditional $\phi$-risk introduced by Bartlett \cite{bartlett06} through 
\begin{equation}
  \label{def:Cab}
    {\bf C}_\phi^{\balpha}(\bbeta)
    = \sum_{l=0}^{2^{m-1}-1}
    (\alpha_l +\alpha_{2^m-1-l}) {\bf C}_\phi^{(\eta_l,1-\eta_l)}(X_l)
  \end{equation}
where $\eta_l = \alpha_l/(\alpha_l +\alpha_{2^m-1-l})$.

\subsection{Existence and uniqueness of a point of minimum of ${\bf C}_\phi^{\alpha}$}
\label{sec:ExistUnique}

Unlike the one-dimensional $\phi$-risk, the cases of unique point of minimum of the multidimensional $\phi$-risk are less simple. 

However, 
a very simple sufficient condition on $\balpha$ (namely $\alpha_j>0, \forall j$) implies existence and uniqueness of the point of minimum of the multidimensional $\phi$-risk under the hypotheses of definition \ref{def-phi} (see lemma \ref{lemma:exist-unique}).

We are then able to identify the sign of the corresponding $X_k$ and the regions of stability of these signs. Recall that these signs are what is needed for constructing the resulting classifier when $\balpha$ is deduced from a truth table. 

\subsection{Regular points}

\begin{definition}[regular points in ${\cal A}^m$]
  Assume $\phi$ satisfies definition \ref{def-phi}. 
  The $2^m$-uplet $\balpha\in{\cal A}^m$ is a regular point for $\phi$ when 
the function ${\bf C}_\phi^{\balpha}$ has a unique point of minimum denoted by $\bbeta^{\min}_\phi(\balpha)$. 
It is the generalization of $\beta^\star(\alpha_0)$.

The functions $\bbeta^{\min}_\phi$ and $X_{\phi,k}^{min}:= \sum_{j=1}^m \epsilon_j (k) (\bbeta^{\min}_\phi)_j, k=0\ldots 2^{m-1}-1$ depend only on $\phi$, $m$ and $\balpha$, and are continuous in $\balpha$.
\label{def:regular}
\end{definition} 

Note that the regularity of a point $\balpha$ is an intrinsic property (it does not depend on $\phi$) and that it is true when all components of $\balpha$ are non zero.

The latter comes from the fact that in this case, the function ${\bf C} _\phi^{\balpha}$ is infinite at infinity.
The former comes from the properties of $\phi$. 

In the case $m=3$, we propose an exhaustive study of all the points of ${\cal A}^3$ in section \ref{sec:labelleprovince}. 
This exhaustive study shows that, for $m=3$, the condition $\balpha$ regular does not depend on $\phi$ satisfying definition \ref{def-phi}.  

This exhaustive study is beyond our reach for larger values of $m$. We can still show that $m+1$ non-zero coefficients at least are needed for the existence and uniqueness of a point of minimum. See lemma \ref{lemma:MinimumCoeff} that follows.

\begin{lemma}
    \label{lemma:MinimumCoeff}
  For all $m\geq 3$, 
  \begin{itemize} 
    \item If less than $m+1$ coefficients are non-zero, there cannot be a unique point of minimum.
    \item  The condition $\alpha_{2^{m}-1}\prod_{k=0}^{m-1}\alpha_{2^k} > 0$ and all others are zero leads to a unique point of minimum. 
    It is the case also for $\alpha_0 \prod_{k=0}^{m-1}\alpha_{2^{m}-1-2^k} >0$.
    These are cases of existence and uniqueness of the point of minimum for exactly $m+1$ non-zero coefficients.
    \item The condition $\alpha_{0}\prod_{k=0}^{m-1}\alpha_{2^k} > 0$ and all others are zero leads to an infimum.
    
    It is the case also for $\alpha_{2^{m}-1} \prod_{k=0}^{m-1}\alpha_{2^{m}-1-2^k} >0$.
  These are also cases of exactly $m+1$ non-zero coefficients. 
  \end{itemize}
\end{lemma}

\begin{proof}\mbox{}

\begin{itemize}
  \item For the proof of the first item, it is enough to notice that $\{ X_j,\alpha_j\neq 0\}$ is not a system of coordinates of $\R^m$ if $\{ j:\alpha_j\neq 0\}$ is of cardinal less than $m-1$, hence there cannot be uniqueness. 

When $\{ j:\alpha_j\neq 0\}$ is of cardinal $m$ and if $\{ X_j,\alpha_j\neq 0\}$ is a system of coordinates of $\R^m$, the function has an infimum equal to $0$. 

When $\{ j:\alpha_j\neq 0\}$ is of cardinal $m$ and if $\{ X_j,\alpha_j\neq 0\}$ is not a system of coordinates of $\R^m$, there cannot be uniqueness. 

\item Owing to $X_0 = \frac{\sum_{k=0}^{m-1} X_{2^k}}{m-2}$ and 
  $X_{2^m-1} = - \frac{\sum_{k=0}^{m-1} X_{2^k}}{m-2}$,
  the function

  $\alpha_{0}\phi(X_{0})+ \sum_{k=0}^{m-1} \alpha_{2^k} \phi(X_{2^k})$
  has all its partial derivatives with respect to $X_{2^k}$ strictly negative, hence no solution for the Euler equations: this function has an infimum which proves the second item. 

\item For the function $\alpha_{2^m-1}\phi(X_{2^m-1})+ \sum_{k=0}^{m-1} \alpha_{2^k} \phi(X_{2^k})$, a possible point of minimum solves  
  $$ F(X_{2^m-1}) := 
  (m-2) X_{2^m-1} + \sum_{k=0}^{m-1} (\phi')^{-1} \left(\frac{\alpha_{2^m-1}}{\alpha_{2^k}(m-2)}\phi'(X_{2^m-1})\right)=0
  $$ 

The limit of $\frac{\alpha_{2^m-1}}{\alpha_{2^k}(m-2)}\phi'(X_{2^m-1})$ is $\frac{\alpha_{2^m-1}}{\alpha_{2^k}(m-2)}l$. If there exists a $k$ such that this limit is $>l$, there exists $X_*$ such that $\frac{\alpha_{2^m-1}}{\alpha_{2^k}(m-2)}\phi'(X_*)=l$, We choose the largest $X_*$ by inspection of all such $k$ and we call $k_0$ one of the corresponding $k$. When $X_{2^{m-1}}\rightarrow X_*$, $X_{2^{k_0}}$ solution of $\phi'(X_{2^{k_0}})=\frac{\alpha_{2^m-1}}{\alpha_{2^{k_0}}(m-2)}\phi'(X_{2^{m-1}})$ goes to $-\infty$, hence $\sum_{k=0}^{m-1}(\phi')^{-1}(\frac{\alpha_{2^m-1}}{\alpha_{2^k}(m-2)}\phi'(X_{2^{m-1}}))\rightarrow -\infty$ hence $F\rightarrow -\infty$ for $X\rightarrow X_*^+$. 

If $\frac{\alpha_{2^m-1}}{\alpha_{2^k}(m-2)}l>l$ for all $k$, then $X_{2^m-1}\rightarrow -\infty$ implies that, for all $k$, $X_{2^k}$ goes to a finite limit. Hence $F \rightarrow -\infty$ when $X_{2^m-1}\rightarrow -\infty$. 

For the other limit $X_{2^m-1}\rightarrow +\infty$,  $F$ goes to $+\infty$. 
As $F$ is continuous, strictly increasing, the existence and uniqueness of the solution of $F(X)=0$ follows. 
\end{itemize}

\end{proof}

Remark that $\{X_{2^k}\}_{0\leq k \leq m-1}$ is a system of coordinates in $\R^m$, hence each $\beta_p$ is a linear combination of $\{X_{2^k}\}_{0\leq k \leq m-1}$. 

\subsubsection{General definitions and properties} 

\begin{proposition}
  \label{prop:point-de-min}
  Assume $\balpha$ is a regular point in ${\cal A}^m$ for $\phi$ regularly classification calibrated and that $\phi$ of class $C^2$.
  
  The functions $\bbeta^{\min}_\phi$ and $X_{\phi,k}^{min}$ are of class $C^1$ on ${\cal A}^m \cap \{\alpha_k>0 \, ,\forall k \}$.
\end{proposition}

\begin{proof}
Assume $\balpha \in {\cal A}^m \cap \{\alpha_k>0 \, ,\forall k \}$.
  There exists a neighborhood of $\balpha$ included in ${\cal A}^m \cap \{\alpha_k>0 \, ,\forall k \}$ such that ${\bf C}_\phi^{\balpha'}$ is strictly convex for $\balpha'$ in this neighborhood, hence existence and uniqueness of the point of minimum of ${\bf C}_\phi^{\balpha'}$ solving the Euler equations, the regularity of $X_{\phi,k}^{min}$ and $\bbeta^{\min}_\phi$ follows from the implicit functions theorem.
\end{proof}

Note that the point of minimum depends on $\phi$, we enforce this property in the notation (unlike the notation
proposed by Bartlett).

\begin{corollary}
    If $\balpha$ is a regular point in ${\cal A}^m$ (definition \ref{def:regular}), there exists a  neighborhood of $\balpha \in \R_+^{2^m}$ in which this definition holds. 
    
    If, for $\balpha$, definition \ref{def:regular} does not hold, there exists a point in any  neighborhood of $\balpha \in \R_+^{2^m}$ for which this definition holds. 
  \end{corollary}
  
  \begin{proof}
    The property of having a unique point of minimum is an open property, that is there exists a neighborhood of a point where this property is true in which it is still true. 

    If, for $\balpha$, definition \ref{def:regular} does not hold, that means that some components of $\balpha$ are zero, hence any neighborhood of $\balpha$ contains a point $\balpha'$ for which all components of $\balpha'$ are non zero: case of existence of an unique minimum.
  \end{proof}

The following definition allows a first glance at classification calibrated functions:

\begin{definition}[properties of pairs]
  \leavevmode
\begin{itemize} 
    \item  For $k\leq 2^{m-1}-1$, the pair $k$ is full iff $\alpha_k\alpha_{2^m-1-k}>0$, the pair $k$ is empty iff $\alpha_k=\alpha_{2^m-1-k}=0$, the pair $k$ is trivial iff $\alpha_k\alpha_{2^m-1-k}=0$ and $\alpha_k + \alpha_{2^m-1-k} > 0$.
    \item A family $k_1,\cdots,k_m$ of $m$ pairs is complete if all these pairs are full and if $\bbeta \rightarrow (X_{k_1},\cdots,X_{k_m})$ is a bijection.  
  \end{itemize}
 \label{def:full-complete}
\end{definition}

\begin{lemma}
  \label{lemma:full-complete}
If at least one pair is full, the infimum of  ${\bf C}_\phi^{\balpha}$ is strictly positive.
\end{lemma}
\begin{proof}
  If a pair $k$ is full, the multidimensional $\phi$-risk 
  ${\bf C}_\phi^{\balpha}$ is always greater than 
  $$(\alpha_k+\alpha_{2^m-1-k})H(\frac{\alpha_k}{\alpha_k+\alpha_{2^m-1-k}}) > 0.$$
  Hence, if the infimum is zero, no pair is full.

 As $\phi>0$, for all $\bbeta$,  ${\bf C}_\phi^{\balpha}(\bbeta) \geq (\alpha_k+\alpha_{2^m-1-k}){\bf C}_\phi^{(\frac{\alpha_k}{\alpha_k+\alpha_{2^m-1-k}},\frac{\alpha_{2^m-1-k}}{\alpha_k+\alpha_{2^m-1-k}})}(X_k)$, hence the result and the consequence on no full pair. 
\end{proof}

A simple criterion for $\balpha$ to be regular (definition \ref{def:regular}) for $\phi$ regularly classification calibrated (definition \ref{def-phi}) is:

\begin{lemma}
  \label{lemma:exist-unique}
  When $\phi$ is regularly classification calibrated (definition \ref{def-phi}), a sufficient condition for ${\bf C}_\phi^{\balpha}$ to have a unique point of minimum on $\{ \bbeta : {\bf C}_\phi^{\balpha}(\bbeta)\leq 1 \}$ is that there exists a sub-family of indexes in $\balpha$ which is complete.
\end{lemma}

\begin{proof}
The proof of this lemma will provide as well the equations for finding the point of minimum $\bbeta(\balpha)$.

The function ${\bf C}_\phi^{\balpha}$ is a function of $X_k,k=0\cdots 2^m-1$.  
The number of independent relations between $X_k, k=0\cdots 2^m-1$ is equal  to $2^m-m$ because $\bbeta \in \R^m$. 
Let $k_1,\cdots,k_m$ a set of indexes of independent coordinates generating full pairs.
Then,  according to (\ref{def:Cab})
\begin{equation}
    {\bf C}_\phi^{\balpha}(\bbeta)
    = \sum_{l=1}^{m}
    (\alpha_{k_l} +\alpha_{2^m-1-k_l}) {\bf C}_\phi^{(\eta_{k_l},1-\eta_{k_l})}(X_{k_l})
    +
    \sum_{k\notin k_1,\cdots,k_m}
    (\alpha_{k} +\alpha_{2^m-1-k}) {\bf C}_\phi^{(\eta_{k},1-\eta_{k})}(X_{k}).
  \end{equation}

The function 
$
( X_{k_1}, \cdots, X_{k_m} ) \rightarrow  \sum_{l=1}^{m} (\alpha_{k_l} +\alpha_{2^m-1-k_l}) {\bf C}_{\phi}^{(\eta_{k_l},1-\eta_{k_l})}(X_{k_l})$ is infinite at infinity on $\R^m$ and is strictly convex on $\R^m$.

If one adds any positive strictly convex function to it, the two properties hold.
Hence, the function ${\bf C}_{\phi}^{\balpha}$ is infinite at infinity hence has a point of global minimum, and is strictly convex, hence this minimum is unique.
This ends the proof of lemma \ref{lemma:exist-unique}.
\end{proof}

Note that one can have a unique point of minimum of the function ${\bf C}_\phi^{\balpha}$ even if $\balpha$ does not satisfy the hypothesis of lemma \ref{lemma:exist-unique} (see the first alinea of \ref{lemma:MinimumCoeff} and theorem \ref{theo:zoology-europe} in the case $m=3$ when $5$ or $6$ coefficients $\alpha_j$ are non zero).

\begin{lemma}
If, on the training set, two classifiers $h_p$ and $h_q$ are equivalent or if one classifier $h_q$ is perfect (definition \ref{def:equivalent-perfect}),  there is not existence and uniqueness of a point of minimum of the function ${\bf C}_\phi^{\balpha}$ and the study reduces to the study of the $m-1$ classifiers $\{ h_1,\cdots,h_m \} \setminus h_q$.
\end{lemma}

Note that, in these two cases, half of the columns in the truth table are zero and no pair is full. 

  \begin{proof}
    In the case $h_p$ and $h_q$ equivalent, there exists a function $K$ such that 
    ${\bf C}_\phi^{\balpha}(\bbeta) = K(\beta_p + \beta_q, {\tilde{\beta}})$ where 
    $\tilde{\beta} = \{ \beta_1,\cdots,\beta_m \} \setminus \{\beta_p,\beta_q\}$. 
    If $K$ has a unique point of minimum on $\R^{m-1}$ then ${\bf C}_\phi^{\balpha}$ has an infinite number of points of minimum.  
    
    If $K$ has an infimum, that is the same for ${\bf C}_\phi^{\balpha}$.

    In the case $h_q$ perfect,
    ${\bf C}_\phi^{\balpha}(\bbeta) = \exp(-\beta_q) {\bf C}_\phi^{\balpha} (\tilde{\beta^q})$ where $\tilde{\beta^q} = \bbeta\vert_{\beta_q=0}$, hence $0$ is the infimum of ${\bf C}_\phi^{\balpha}$ ($\beta_q \rightarrow +\infty$).
  \end{proof}

\section{The optimal combination of classifiers}
\label{sec:optimal-classifier}

We then provide a rigorous accurate analysis of the optimal combination of these classifiers for any convexified cost function\footnote{Note that this is the situation described in \cite{bach2024learning} (p.303) for $\phi = 1/\exp$.} satistifying definition \ref{def-phi}. We obtain existence and uniqueness results.

A resulting classifier is 
$x \rightarrow \sgn\sum_{j=1}^m \beta_j h_j(x)$ for any choice of $\bbeta$.

The convexified empirical risk of false decision is (equation (\ref{mBartlettSH})) : 

\centerline{${\cal R}^{\cS,{\cal H}_m}_{\phi}(\bbeta) = \sum_{k=0}^{2^{m-1}-1}(\alpha_k^{\cS, \cH_m}\phi(X_k)+\alpha_{2^m-1-k}^{\cS, \cH_m}\phi(-X_k))$.}

We are now ready to construct the optimal classifier associated with $\cS, {\cal H}_m$ for each function $\phi$ satisfying definition \ref{def-phi}.

\begin{definition}[regular basis of classifiers]
The classifiers $\cH_m$ form a regular basis of classifiers for $\cS$ if and only if $\balpha^{\cS,\cH_m}$ is a regular point (see definition \ref{def:regular}). 

The value of $X_k$ associated with this unique point of minimum is denoted by
$X_{\phi,k}^{\min}(\balpha^{\cS,\cH_m})$,  $k=0\ldots 2^{m-1}-1$. 
\label{stanislas-1}
\end{definition} 

\begin{remark}
  If $\cS' \supset \cS$ and if $\cH_m$ form a regular basis for $\cS$, it forms a regular basis for $\cS'$.
  In other words, if for a training set and a list of classifiers we obtain existence and uniqueness of a point of minimum of the empirical convex risk, we can add any number of examples without destroying this property but changing the point of minimum.
\end{remark}
Note that the notion of regular basis of classifiers is a notion associated with $\cH$ and $\cS$ while the notion of regular point $\balpha$ (definition \ref{def:regular}) is not, it is only an abstract property of $\balpha$.

\begin{theorem}
  \label{theo:main}
Adopt the convention $\sgn(0)=1$. 

When $\balpha^{\cS,\cH_m}$ is a regular point of ${\cal A}^m$, the optimal classifier (depending on $\phi$) is 
$h_\star = \sgn(\sum_{j=1}^m \beta_{\phi,j}^{min} h_j)$, which gives the decision 
$\sgn(X_{\phi,k}^{\min}(\balpha^{\cS,\cH_m}))$ for $k=0,\ldots,2^m-1$.
\end{theorem} 

\begin{proof}
  This theorem is a straightforward consequence of proposition \ref{prop:point-de-min}.
\end{proof}

We highlight that the case of three classifiers, studied thoroughly in theorem \ref{theo:zoology-europe}, present some interesting features:
\begin{remark}
  Having a complete sub family is a sufficient condition for existence and uniqueness of the point of minimum. It is not a necessary condition: see theorem \ref{theo:zoology-europe}, item (iii).
\end{remark}

\begin{remark}
  There are cases where it is impossible to construct an optimal classifier whatever the convexification is.
  See theorem \ref{theo:zoology-europe}, items (ii), (iii), (iv).
\end{remark}  

\begin{remark}
  Note that the resulting classifier $\sgn(\sum \beta_j G_j)$ does not change if one multiplies $\bbeta$ by $\lambda>0$, hence even the cases of infima can return a "limit resulting classifier" by considering the limit $\bbeta^\star$ of $\frac{\bbeta_l}{||\bbeta_l||}$ when $\bbeta_l$ is a sequence such that ${\cal R}^{\cS,{\cal H}_m}_{\phi}(\bbeta_l) \rightarrow \mbox{inf } {\cal R}^{\cS,{\cal H}_m}_{\phi}$.
  However, $\bbeta^\star$ does not keep all the information contained in the truth table. Indeed, the limit value $X^\star$ associated with $\bbeta^\star$ is $\delta_{j i_0}$ (all coordinates are zero except one). 
  It is a corner point (see definition \ref{def:frontiers} below). 
\end{remark}

\section{All you want to know on the case of at most three classifiers}
\label{sec:labelleprovince}

We describe in a first subsection the combination of two classifiers, using both the $0/1$ risk and the convexified one, and show that, in the case of two classifiers, the optimal combination yields the best classifier which can be the opposite of one of two classifiers considered.

\subsection{The case of two classifiers}

We consider now two classifiers $h_1$ and $h_2$ and study the classifier ${\rm sign}(\beta_1 h_1 + \beta_2 h_2)$.
Note that, for each example $(x_i,y_i)\in \cS$, we have

$y_i \, {\rm sign}[(\beta_1 h_1 + \beta_2 h_2)(x_i)] = 
{\rm sign}[\beta_1 y_i h_1(x_i) + \beta_2 y_i h_2(x_i)] =
{\rm sign}[\beta_1 G_1(x_i,y_i) + \beta_2 G_2(x_i,y_i)]$.

Since $y_j\in\{\pm 1\}$ and $h_j(\cdot)\in\{\pm 1\}$, we have $G_j(x_i,y_i)\in\{\pm 1\}$ and $(\beta_1 G_1+\beta_2 G_2)(\cdot) \in \{\pm\beta_1 \pm \beta_2\}$.

\subsubsection*{0/1 loss (Logical risk)} 

\begin{lemma}[Properties of the logical risk for two classifiers]
\label{lem:logical-2}
  The resulting classifier is the best of the four classifiers ($h_1$, $h_2$, $-h_1$, $-h_2$).
\end{lemma}

Proof in section \ref{sec:logical-2}.

\subsubsection*{Convexified risk}

\begin{lemma}[Properties of the convexified risk for two classifiers]
  \label{lemma:eu}
\begin{enumerate}[label=(\roman*)]
  \item The convexified cost function has a unique point of minimum if and only if no column of the truth table is empty.
  
  This point of minimum, given in appendix \ref{sec:logical-2}, returns a resulting classifier which gives the same answer on the set of examples $\cS$ as for the $0/1$ risk.
  
  This resulting classifier ${\rm sign}(\beta_{1,\phi}^{min} h_1 + \beta_{2,\phi}^{min} h_2)$ depends on $\phi$ hence the risk of generalization for new examples depends on $\phi$.
  \item If only one column is empty, the convexified cost function has a strictly positive infimum.
  \item If two columns only are empty:
  \begin{enumerate}
    \item and if they do not belong to the same pair, the convexified cost function has zero as infimum. 
    \item and if they belong to the same pair, the convexified cost function has a non-unique minimum strictly positive.
  \end{enumerate}
\end{enumerate} 
\end{lemma}

\begin{proof}
  The convexified cost function is
  $$\alpha_{0}^{\cS, \cH_2} \phi(X_0) + \alpha_{3}^{\cS, \cH_2} \phi(-X_0) + \alpha_{1}^{\cS, \cH_2} \phi(X_1) + \alpha_{2}^{\cS, \cH_2} \phi(-X_1).$$

\begin{enumerate}[label=\roman*.]  
  \item  If no column is empty, 
  
  $\alpha_{0}^{\cS, \cH_2} \phi(X_0) + \alpha_{3}^{\cS, \cH_2} \phi(-X_0)$ has a unique point of minimum $X_{\phi,0}^{\min}(\balpha^{\cS,\cH_2})$ and $\alpha_{1}^{\cS, \cH_2} \phi(X_1) + \alpha_{2}^{\cS, \cH_2} \phi(-X_1)$ has a unique point of minimum $X_{\phi,1}^{\min}(\balpha^{\cS,\cH_2})$. 
The minimum of ${\cal R}^{\cS,{\cal H}_m}_{\phi}$ is strictly positive.

  \item If only one column is empty, two cases:
  
  \begin{itemize}
  \item If $\alpha_{0}^{\cS, \cH_2}$ or $\alpha_{3}^{\cS, \cH_2}$ is equal to zero, $X_{\phi,1}^{\min}(\balpha^{\cS,\cH_2})$ is uniquely determined and $X_0$ goes to infinity. 
  The value of the infimum is $\alpha_{1}^{\cS, \cH_2} \phi(X_1) + \alpha_{2}^{\cS, \cH_2} \phi(-X_1)$ at $X_{\phi,1}^{\min}(\balpha^{\cS,\cH_2})$ and is strictly positive.
  \item If $\alpha_{1}^{\cS, \cH_2}$ or $\alpha_{2}^{\cS, \cH_2}$ is equal to zero, $X_{\phi,0}^{\min}(\balpha^{\cS,\cH_2})$ is uniquely determined and $X_1$ goes to infinity. 
   The value of the infimum is $\alpha_{0}^{\cS, \cH_2} \phi(X_0) + \alpha_{3}^{\cS, \cH_2} \phi(-X_0)$ at $X_{\phi,0}^{\min}(\balpha^{\cS,\cH_2})$ and is strictly positive.
\end{itemize}

\item If two columns are empty, two cases (depending on which the two columns belong to the same pair or not):
\begin{enumerate}[label=(\alph*)]
  \item If $\alpha_{1}^{\cS, \cH_2}=\alpha_{3}^{\cS, \cH_2}=0$, the infimum of the function is zero (when $X_1\rightarrow +\infty$ and $X_0\rightarrow -\infty$). This is also the case when $\alpha_{0}^{\cS, \cH_2}=\alpha_{2}^{\cS, \cH_2}=0$.
  \item If $\alpha_{0}^{\cS, \cH_2}=\alpha_{3}^{\cS, \cH_2}=0$, $X_{\phi,1}^{\min}(\balpha^{\cS,\cH_2})$ is uniquely determined and the function does not depend on $X_0$, hence there is a non-unique strictly positive minimum. This is also the case when $\alpha_{1}^{\cS, \cH_2}=\alpha_{2}^{\cS, \cH_2}=0$.
\end{enumerate}

\end{enumerate}
\end{proof}

\begin{lemma}
  \label{lem:tikhonov}
  When a column is empty, there is an infimum (see lemma \ref{lem:logical-2}). The construction of a new classifier is undecidable. It corresponds to a situation on the frontier (see section \ref{sec:zoll}). 
\end{lemma}

\begin{proof}
  When there is an infimum, it is not mathematically possible to construct a unique solution of the optimization problem, hence a unique optimal classifier.
  Traditionally, one uses a regularization to overcome the problem.  

Three strategies are possible : 

1 - The first strategy is to consider the limit of the point 
$\frac{(X_0, X_1^*)}{\vert\vert (X_0, X_1^*)\vert \vert}$ when $X_0\rightarrow +\infty$. This limit is $(1,0)$ and thus the limit of $\beta$ is $(-\frac 12, \frac 12)$. However, one verifies that the sign of $X_1^*$ determines a decision, information which is no longer present in the direction $(-\frac 12, \frac 12)$ (and the direction $(\frac 12, -\frac 12)$ is equally possible).

2 - The second strategy is the regularization of the truth table. In this case one can either consider adding $\epsilon$ to the empty column of the truth table or adding $\epsilon$ to all elements of the truth table. The function we seek the (unique) minimum of is

$$\alpha_{0}^{\cS, \cH_2} \phi(X_0) +\epsilon \phi(-X_0)+ \alpha_{1}^{\cS, \cH_2} \phi(X_1) + \alpha_{2}^{\cS, \cH_2} \phi(-X_1).$$

The unique point of minimum reads 
$$X_0^*= \beta^\star( \frac{\alpha_{0}^{\cS, \cH_2} }{\alpha_{0}^{\cS, \cH_2}+\epsilon}), X_1^*= \beta^\star(\frac{\alpha_{1}^{\cS, \cH_2} }{ \alpha_{1}^{\cS, \cH_2} + \alpha_{2}^{\cS, \cH_2} }),$$
Its limit when $\epsilon\rightarrow 0$ is then $(\beta^\star(1), X_1^*)=(+\infty, X_1^*)$ and for all $\epsilon>0$, it provides a value of $(\beta_1, \beta_2)$.\\
When adding $\epsilon$ to all columns of the truth table, one obtains
$$X_0^*= \beta^\star( \frac{\alpha_{0}^{\cS, \cH_2} +\epsilon}{\alpha_{0}^{\cS, \cH_2}+2\epsilon}), X_1^*= \beta^\star(\frac{\alpha_{1}^{\cS, \cH_2} +\epsilon}{ \alpha_{1}^{\cS, \cH_2} + \alpha_{2}^{\cS, \cH_2} +2\epsilon}).$$
This is an equivalent strategy, the limit is identical, as well as the resulting signs of $X_0^*, X_1^*$.

3 - The third strategy is to consider a quadratic regularization (called Tikhonov regularization). 
Then one studies
$$\alpha_{0}^{\cS, \cH_2} \phi(X_0) + \alpha_{1}^{\cS, \cH_2} \phi(X_1) + \alpha_{2}^{\cS, \cH_2} \phi(-X_1)+\frac12\epsilon (X_0^2+X_1^2).$$
One obtains the Euler equations on the point of minimum of
$$\alpha_{1}^{\cS, \cH_2} \phi(X_1) + \alpha_{2}^{\cS, \cH_2} \phi(-X_1)+\frac 12\epsilon X_1^2$$
(which converges to $X_1^*$ when $\epsilon \rightarrow 0_+$)
and on the point of minimum of $\alpha_{0}^{\cS, \cH_2} \phi(X_0) +\frac 12\epsilon X_0^2.$

For this point, the Euler equation is
\[
\varepsilon X_0 + \alpha_{0}^{\cS, \cH_2} \varphi'(X_0) = 0 \iff \frac{X_0}{\varphi'(X_0)} = -\frac{\alpha_{0}^{\cS, \cH_2}}{\varepsilon}
\]

\[
\frac{\alpha_{0}^{\cS, \cH_2}}{\varepsilon} \in \mathbb{R}^*_-, \ X_0 < 0, \ \varphi'(X_0) < 0 \implies \frac{X_0}{\varphi'(X_0)} < 0 \quad \text{for } X_0 < 0, 
\]

hence solutions are in $(0, +\infty)$.

For $X_0 \to +\infty$, $\varphi'(X_0) \to 0_-$, $\frac{X_0}{\varphi'(X_0)} \to -\infty$.

For $X_0 \to 0$, $\frac{X_0}{\varphi'(X_0)} \to 0$, there exists at least a solution of the Euler equation (and it is unique thanks to coerciveness) denoted by $X_0(-\alpha_{0}^{\cS, \cH_2}/\varepsilon)$, which goes to $+\infty$ when $\varepsilon \to 0_+$.

In all these strategies, after normalization we obtain a resulting classifier where one of the values of $X_j$ is close to zero. 
In this particular case, the result is independent of the classification calibrated function. 
\end{proof}

We observe already that the situation of the conditional $\phi$-risk in $\R^2$ is more complicated than the one defined by Bartlett {\it et al} \cite{bartlett06}: cases of infimum and cases of no unique minimum are present when 
$\alpha_{0}^{\cS, \cH_2}\alpha_{1}^{\cS, \cH_2}\alpha_{2}^{\cS, \cH_2}\alpha_{3}^{\cS, \cH_2}=0$.

More interesting results for the type  of structure seen above appear with at least $3$ classifiers: first of all, the necessary and sufficient condition to define $\bbeta^{\min}_\phi (\balpha)$ is more complicated than $\alpha_{0}\cdots\alpha_{7}>0$ and this point of minimum does not return always the same answer on the set of examples $\cS$ as for the logical risk, the optimal combination of the classifiers depends on $\phi$.

In the rest of this section, dedicated to the case of three classifiers, we perform a complete study of all the cases where there is a unique minimum, a non unique minimum and an infimum for the generical conditional $\phi$-risk for any $\phi$ and, in the case of two particular functions $\phi$, we calculate the point of minimum and we derive the equations of the {\it frontiers} (see definition \ref{def:frontiers}).

\subsection{The case of three classifiers}
\label{sec:setup-with-3-classifiers}

\label{sec:existence-uniqueness-3}

\subsubsection*{0/1 loss (Logical risk)}
In the case of three classifiers, the empirical risk is:
$$
{\cal R}^{\cS,{\cal H}_3}_{\mathbbm{1}_{\bullet <0}}(\bbeta)
= \sum_{j=0}^{3} 
\alpha_j^{\cS,{\cal H}_3} \mathbbm{1}_{X_j< 0} + \alpha_{7-j}^{\cS,{\cal H}_3} \mathbbm{1}_{X_j \geq 0}.
$$

To each value of $\bbeta = \left(\beta_1,\beta_2,\beta_3\right)\in \R^3$  corresponds the signs of $X_i$, 
thus the value of ${\cal R}^{\cS,{\cal H}_3}_{\mathbbm{1}_{\bullet <0}}(\bbeta)$  in a finite list. 

Each element of the list corresponds to $\bbeta$ in a region of $\R^3$, which is the intersection of $4$ half-spaces, because the empirical risk is piecewise constant, an infinite number of values of $\bbeta$ returns the minimum value.

\subsubsection*{Convexified risk} 

The empirical convexified risk is
\begin{equation}
{\cal R}^{\cS,{\cal H}_3}_{\phi}(\bbeta) 
= \sum_{j=0}^{3} \alpha_j^{\cS,{\cal H}_3} \phi(X_j)+ \alpha_{7-j}^{\cS,{\cal H}_3}\phi(-X_j).
  \label{eq:gtcr}
  \end{equation}
This is the generic tridimensional conditional $\phi-$risk where
$X_0 =-\beta_1-\beta_2-\beta_3, X_1 =-\beta_1-\beta_2+\beta_3,
X_2 =-\beta_1+\beta_2-\beta_3, X_3 =-\beta_1+\beta_2+\beta_3$.

\vspace{\baselineskip}

The next two subsections describe the properties of the conditional $\phi$-risk in $\R^3$ which is the relevant object to study the optimal combination of $3$ classifiers. 

\vspace{\baselineskip}

We first derive analytic formulae for the computation of the point of minimum $\bbeta_{\rm logit}^{min}(\balpha)$  when $\phi$ is the function logit, a single transcendental equation for the computation of the point of minimum $\bbeta_{\rm boost}^{min}(\balpha)$  when $\phi$ is the function boost.

\vspace{\baselineskip}

Secondly, we describe the necessary and sufficient condition to obtain a unique point of minimum of the conditional $\phi$-risk in $\R^3$, as well as the necessary and sufficient condition to get a non unique minimum (and of course all the cases of point of infimum can be deduced).

\subsection{Formulae for the point of minimum for the generalized convexified $\phi$-risk for boost and logit loss functions}
\label{sec:point-of-minimum-and-frontiers-boosting-logit}

The points of minimum are accounted for through a scalar equation which can even be solved analytically for the function logit.

Let 

$A=\alpha_0\alpha_3-\alpha_4\alpha_7+\alpha_0\alpha_5-\alpha_2\alpha_7+\alpha_0\alpha_6-\alpha_1\alpha_7+\alpha_3\alpha_5-\alpha_2\alpha_4+\alpha_3\alpha_6-\alpha_1\alpha_4+\alpha_5\alpha_6-\alpha_1\alpha_2$,

$B=\alpha_0\alpha_3\alpha_5+\alpha_0\alpha_3\alpha_6+\alpha_3\alpha_5\alpha_6+\alpha_1\alpha_2\alpha_4+\alpha_1\alpha_2\alpha_7+\alpha_2\alpha_4\alpha_7
+\alpha_0\alpha_5\alpha_6+\alpha_1\alpha_4\alpha_7$,

$C=\alpha_0\alpha_3\alpha_5\alpha_6-\alpha_1\alpha_2\alpha_4\alpha_7$, and define the function $\Gamma_l(y,\balpha)$ by 
\begin{equation}
  \Gamma_l(y,\balpha) = y^3 + A y^2 + B y + C
  \label{eq:Gammal}
\end{equation}

Let the function $\Gamma_b(x,\balpha)$ be given by 
\begin{equation}
  \Gamma_b(x,\balpha) = \prod_{j=0}^3 \left( \frac{x}{2} + \sqrt{\left(\frac{x}{2}\right)^2 + \alpha_j\alpha_{7-j}}\right).
  \label{eq:Gammab}
\end{equation}

\begin{lemma}
Assume that all the $\alpha_j$ are non zero. The functions $\Gamma_b$ and $\Gamma_l$ are strictly increasing on their interval of definition. \\
The equations $\Gamma_l(y,\balpha) = 0$ and $\Gamma_b(x,\balpha) = \alpha_0 \alpha_3 \alpha_5 \alpha_6$ have, respectively in their sets of definition, a unique solution denoted by $y_l(\balpha)$ and $x_b(\balpha)$. 

One has even an analytic\footnote{Note for example that for 
$\balpha = (2/5, 1/20, 1/20, 1/20, 3/20, 1/20, 1/20, 1/5)$ 
the Cardan Formula for a positive discriminant returns the exact numerical value of the root of the polynomial $\Gamma_l$. 
Some cases need to go through complex values to calculate the $3$ real roots of a polynomial of degree $3$.} formula for $y_l(\balpha)$ using Cardan-Ferrari formulae. 
\end{lemma}
\begin{proof}
It comes from the monotony of $\Gamma_b$ on  $\mathbb R$ and the limits $0_+$ and $+\infty$ at $-\infty$ and $+\infty$, and the monotony of $\Gamma_l$ on 
$I = (-\mbox{min}(\alpha_6, \alpha_5, \alpha_3, \alpha_0), 
\mbox{min}(\alpha_1, \alpha_2, \alpha_4, \alpha_7))$, with image $\R$,
the latter coming from the following equivalent equation for $y_l(\balpha)$: 
$\frac{\alpha_1-y}{\alpha_6+y}\frac{\alpha_2-y}{\alpha_5+y}\frac{\alpha_4-y}{\alpha_3+y}   \frac{\alpha_7-y}{\alpha_0+y} = 1$.
\end{proof}

\begin{theorem}
\label{equations-minimum-frontiers:lb}
In the cases of logit and boost, for $\alpha_0 \cdots \alpha_7 > 0$, the unique point of minimum of the generalized convexified $\phi$-risk is given by the following formulae:
\begin{itemize}
  \item In the case of boost, the unique point of minimum is 
  $\bbeta_b(\balpha)= \frac12(-X^b_2(\balpha)-X^b_1(\balpha), X^b_3(\balpha)-X^b_1(\balpha), X^b_3(\balpha)-X^b_2(\balpha))$ where
  
  $$
  \bs\ba{l}X_1^b(\balpha)=\ln (\frac{x_b(\balpha)}{2\alpha_6}+\sqrt{(\frac{x_b(\balpha)}{2\alpha_6})^2+\frac{\alpha_1}{\alpha_6}})\\
  X_2^b(\balpha)= \ln (\frac{x_b(\balpha)}{2\alpha_5}+\sqrt{(\frac{x_b(\balpha)}{2\alpha_5})^2+\frac{\alpha_2}{\alpha_5}})\\
  X_3^b(\balpha)=\ln (-\frac{x_b(\balpha)}{2\alpha_4}+\sqrt{(\frac{x_b(\balpha)}{2\alpha_4})^2+\frac{\alpha_3}{\alpha_4}})\ea\es.
  $$
   
\item In the case of logit, the unique point of minimum is

 $\bbeta_l(\balpha)= \frac12(-X^l_2(\balpha)-X^l_1(\balpha), X^l_3(\balpha)-X^l_1(\balpha), X^l_3(\balpha)-X^l_2(\balpha))$ where 

 $$
 \bs\ba{l}X_1^l(\balpha)=\ln \frac{\alpha_1-y_l(\balpha)}{\alpha_6+y_l(\balpha)}\\
  X_2^l(\balpha)=\ln \frac{\alpha_2-y_l(\balpha)}{\alpha_5+y_l(\balpha)}\\
  X_3^l(\balpha)=\ln \frac{\alpha_3+y_l(\balpha)}{\alpha_4-y_l(\balpha)}\ea\es.
  $$
\end{itemize}
\end{theorem}

The proof is given in appendix \ref{proof:equations-minimum-frontiers:lb}.
\begin{remark} 
  \label{remark:limit_infimum2}
  When $\alpha_1\alpha_3\alpha_4\alpha_5\alpha_6\alpha_7 > 0$ and $\alpha_0=\alpha_2=0$, the convexified cost has an infimum. If one uses the second strategy described in the proof of lemma \ref{lem:tikhonov}, we consider the function $C^{\balpha, reg}_{\phi}:$ $\bbeta\rightarrow C^{\balpha}_{\phi}(\bbeta)+\kappa(\phi(X_0)+\phi(X_2))$ where $\kappa>0$. This function has an unique point of minimum $\bbeta_{\kappa}$ with $X_{0, \kappa}=O(\ln \kappa), X_{1, \kappa}= \ln \frac{\alpha_1}{\alpha_6}+O(\kappa), X_{2, \kappa}=O(\ln \kappa), X_{3, \kappa}=\ln\frac{\alpha_3}{\alpha_4}+O(\kappa)$ and the point $X_{\kappa}$ has not a finite limit when $\kappa\rightarrow 0_+$. 
  This result holds also if one considers $\balpha^\star = \frac{1}{1+8\kappa} (\balpha + (\kappa,\cdots,\kappa))$. 
\end{remark}
\begin{proof}
  When $\alpha_1\alpha_3\alpha_4\alpha_5\alpha_6\alpha_7 > 0$, denote by ${\tilde \balpha}=\frac{1}{1+2\kappa}(\kappa, \alpha_1, \kappa, \alpha_3, \alpha_4, \alpha_5, \alpha_6, \alpha_7)$. There exists $\kappa_0 > 0$ such that, for $\kappa<\kappa_0$
  $ (-\mbox{min}({\tilde \alpha}_6, {\tilde \alpha}_5, {\tilde \alpha}_3, {\tilde \alpha}_0), 
\mbox{min}({\tilde \alpha}_1, {\tilde \alpha}_2, {\tilde \alpha}_4, {\tilde \alpha}_7)) \subset (-\kappa, \kappa)$. 
  Recall that the polynomial $\Gamma_l$ when $\alpha_0=\alpha_2=0$ is $\Gamma_l(y)=y^3+Ay^2+By$. Denote by  $\Gamma_l^{\kappa}$ the one derived for $\alpha_0=\alpha_2=\kappa$ is $y^3+A(\kappa)y^2+B(\kappa)y+\frac{\kappa(\alpha_3\alpha_5\alpha_6-\alpha_1\alpha_4\alpha_7)}{(1+2\kappa)^4}$, where $A(\kappa)= \frac{A+\kappa(\alpha_3+\alpha_5+\alpha_6-\alpha_1-\alpha_4-\alpha_5)}{(1+2\kappa)^2}$ and  $B(\kappa)= \frac{B+\kappa(\alpha_3\alpha_5+\alpha_3\alpha_6+\alpha_1\alpha_4+\alpha_4\alpha_7+\alpha_5\alpha_6)}{(1+2\kappa)^3}$. hence $\Gamma_l^{\kappa}$ has a unique root of order $\kappa$ which belongs to $(-\kappa, \kappa)$, denoted by $\kappa z_*(\kappa)$, this root being the value of $\alpha_7e^{X_{0, \kappa}}/(1+e^{X_{0, \kappa}})$.  Hence $X_0=O(\ln \kappa)$ goes to $-\infty$ when $\kappa\rightarrow 0$. The relations $e^{X_{1, \kappa}}=\frac{\alpha_1-\kappa z_*(\kappa)}{\alpha_6+\kappa z_*(\kappa)}$, $e^{X_{3, \kappa}}= \frac{\alpha_3-\kappa z_*(\kappa)}{\alpha_4+\kappa z_*(\kappa)}$, and $e^{X_{2, \kappa}}= \frac{\kappa -\kappa z_*(\kappa)}{\alpha_5+\kappa z_*(\kappa)}$ yield the limit of $X_{\kappa}$ announced. 
  The same formulae yield an identical result for $\balpha^\star$. 
  This regularization yields a result but does not yield a bounded value of $\bbeta$ at the limit.
  Remark that $\frac{\bbeta(\kappa)}{||\bbeta(\kappa)||}\rightarrow (-\frac{1}{2},0,-\frac{1}{2},0)$ as $\kappa\rightarrow 0$. However, the classifier resulting classifier $G= \sum \beta_j(\kappa) G_j$ is not a separator: the signs in columns $1$ and $6$ are oppposite hence, for every small $\kappa$,  neither $G$ or $-G$ is perfect.
\end{proof}

\begin{remark}
  We used here two Tikhonov-type regularizations adapted to our problem instead of the usual $\ell^2$ regularization on the norm of the vector $\bbeta$; the analysis is also possible with the usual regularization and the conclusion is identical.
\end{remark}

\subsection{Non-existence or non-uniqueness of the point of minimum for a general conditional $\phi$-risk in $\R^3$}

In this paragraph, we consider $\balpha\in {\cal A}^3$ (cf. \eqref{eq:domain}) and describe all the possible cases for non-existence or non-uniqueness for all loss functions satisfying definition \ref{def-phi} (as well as cases where the minimum is unique). 

Before describing all the cases, we would like to emphasise that the convexified risk does not necessarly have a unique point of minimum. There can be a point of infimum if $6$ (out of $8$) entries of $\balpha$ are non zero, let us give an example:
\begin{remark}
  \label{rem:inf6}
  Let $\alpha_1\alpha_3\alpha_4\alpha_5\alpha_6\alpha_7>0$ and $\phi$ be an arbitrary function satisfying definition \ref{def-phi}, the generic multidimensional $\phi$-risk  
  $$\alpha_1 \phi(X_1) + \alpha_6 \phi(-X_1) + \alpha_3 \phi(X_3) + \alpha_4 \phi(-X_3) + \alpha_5 \phi(X_1-X_0-X_3) + \alpha_7 \phi(-X_0)$$

  has a strictly positive infimum, hence there is  no combination of the $3$ classifiers which minimizes the convexified risk. 
\end{remark}

\begin{proof}
  The partial derivative with respect to $X_0$ of this function is 
  $- \alpha_5\phi'(X_1-X_0-X_3) - \alpha_7\phi'(X_0)$ which can never be zero. Hence the Euler equations do not have a solution. The infimum is equal to 
  
  $\mbox{min}(\alpha_1 \phi(X_1) + \alpha_6 \phi(-X_1))
  +
 \mbox{min} (\alpha_3 \phi(X_3) + \alpha_4 \phi(-X_3))> 0$.
\end{proof}

It is then interesting to investigate the cases where the convexified cost does not have a unique point of minimum, because in this case no algorithm can provide a solution to the minimization of the convexified cost function. Nevertheless, the implementation of an algorithm converges numerically (using a stopping criterion), hence returns a point.

We provide a detailed analysis of all the cases where we have regular (cf. definition \ref{def:regular}) points in $ {\cal A}^3$ (cf. \eqref{eq:domain}). \\

\begin{theorem} 
  \label{theo:zoology-europe}
  We consider any loss function satisfying definition \ref{def-phi}.
  \begin{enumerate}[label=(\roman*)]
    \item When $3$ pairs are full and $0$, $1$ or $2$ non zero elements in the fourth one, the point of minimum exists and is unique. 
    \item When  $1$, $2$ or $3$ entries of $\balpha$ only are non-zero, there is an infimum (equal to zero) or a non-unique minimum (the latter occurs for only two non-zero entries, and in the same pair).
    \item The only cases of exactly $4$ non-zero coefficients where there is existence and uniqueness of the point of minimum are $\alpha_0\alpha_3\alpha_5\alpha_6>0$ or $\alpha_1\alpha_2\alpha_4\alpha_7>0$. 
    All other cases for exactly $4$ non-zero coefficients return a zero infimum or a non-unique non-zero minimum.
    \item For $5$ or $6$ non-zero coefficients and no empty pair, the multidimensional $\phi$-risk is bounded below by a strictly positive constant and
    the only cases of having a unique point of minimum are one of the previous cases $\alpha_0\alpha_3\alpha_5\alpha_6>0$ or $\alpha_1\alpha_2\alpha_4\alpha_7>0$ with any one or two additional elements.
    \item For $5$ non zero coefficients and one empty pair, the function has a non-zero infimum.
  \end{enumerate}
  The cases of an infimum equal to zero are obtained only in items (ii) and (iii).
\end{theorem}

The following corollary is a consequence of theorem \ref{theo:zoology-europe} by inspection of (iii) and (iv):
\begin{corollary} 
  \label{theo:zoology-world}
We consider any loss function satisfying definition \ref{def-phi}.
When no pair is empty, there exists a unique point of minimum if and only if $\alpha_0 \alpha_3\alpha_5\alpha_6 > 0$ or $\alpha_1\alpha_2\alpha_4\alpha_7 > 0$.  
\end{corollary}

\begin{proof}
Items (i),(ii),(iv) are proven in appendix \ref{sec:proof-theo:zoology-europe}.

All cases of four entries non zero (item (iii)) are described here:
\begin{enumerate} 
\item two pairs full: existence of multiple points of minima
\item one pair full and two non zero entries in two of the three others: infimum. For example $\alpha_0\alpha_7\alpha_4\alpha_2>0$, for which $X_0$ solves $\alpha_0 \phi'(X_0) = \alpha_7 \phi'(-X_0)$, $X_3 \rightarrow -\infty$, $X_2 \rightarrow +\infty$ and the infimum is $\min [ \alpha_0 \phi(X_0) + \alpha_7 \phi(-X_0)]$. 
\item each pair contains a non zero element: 
\begin{enumerate}
\item The two cases of minimum are $\alpha_7\alpha_4\alpha_2\alpha_1>0$ and $\alpha_0\alpha_3\alpha_5\alpha_6>0$ (It is a particular case of lemma \ref{lemma:MinimumCoeff}(i)). 

In the first case (the second case is equivalent by changing $\bbeta$ to  $-\bbeta$), 
the function is $\alpha_7\phi(-X_1-X_2-X_4)+\alpha_1\phi(X_1)+\alpha_2\phi(X_2)+\alpha_4\phi(X_4)$. The Euler equations are 
$\alpha_7\phi'(-X_0)= \alpha_4\phi'(X_4),
\alpha_7\phi'(-X_0)= \alpha_2\phi'(X_2),
\alpha_7\phi'(-X_0)= \alpha_1\phi'(X_1)$. 
From the equality $X_0 = X_1 + X_2 + X_4$ and the values of $X_1,X_2,X_4$ in term of $X_0$, we get the equation $Z(X_0)=0$ where $Z(X_0) = X_0\rightarrow X_0- (\phi')^{-1}[\frac{\alpha_7}{\alpha_1}\phi'(-X_0)]- (\phi')^{-1}[\frac{\alpha_7}{\alpha_2}\phi'(-X_0)]- (\phi')^{-1}[\frac{\alpha_7}{\alpha_4}\phi'(-X_0)]$.
As $\phi'$ is strictly increasing, its reciprocal is strictly increasing and $Z$ is strictly increasing. 

When $X_0\rightarrow -\infty$ its limit is $-\infty$ because $ \frac{\alpha_7}{\alpha_k}\phi'(-X_0)\rightarrow 0_-$  for all $k$ hence  

$ (\phi')^{-1}[\frac{\alpha_7}{\alpha_k}\phi'(-X_0)]\rightarrow +\infty$ for all $k$. 
 
When $X_0\rightarrow +\infty$, assume that $l$ is the limit (element of $[-\infty, 0)$) of $\phi'$ at $-\infty$. When $X_0\rightarrow +\infty$, $\frac{\alpha_7}{\alpha_k}\phi'(-X_0)\rightarrow \frac{\alpha_7}{\alpha_k}l$. 

The limitation (for the set of definition of $(\phi')^{-1}[.]$) for $X_0$ is $\frac{\alpha_7}{\alpha_k}\phi'(-X_0)>l$. Hence, considering the smallest value of $\frac{\alpha_k}{\alpha_7}$, one gets $X_0^*$ such that  $\phi'(-X_0)= \mbox{min}(\frac{\alpha_k}{\alpha_7}l)$ (and $X_0^*=-\infty$ when $l=-\infty$), and the set of definition of $Z$ is $(-\infty,X_0^*)$, the limit when $X_0\rightarrow X_0^*$ is $+\infty$. As $Z$ is strictly increasing, there is a unique root to $Z(X_0)=0$, hence a unique solution for the Euler equations which is the unique point of minimum of the convexified risk. 

Remark that, for the boost loss, the value of the minimum is $(\alpha_0 \alpha_1 \alpha_2 \alpha_4)^{1/4}$.

\item All other cases return a zero infimum:

we consider the specific example $\alpha_7\alpha_4\alpha_2\alpha_6>0$ for the proof (the proof of all other cases is similar). The function is 
$\alpha_7\phi(-X_1-X_2+X_3)+\alpha_4\phi(-X_3)+\alpha_2\phi(X_2)+\alpha_6\phi(-X_1)$.
Its derivative with respect to $X_1$ is $-\alpha_6\phi'(-X_1)-\alpha_7\phi'(-X_0)>0$, hence it has no root and there is no point of minimum. 

With $X_1 = -3n, X_2 = n, X_3 = -n \rightarrow -X_1-X_2+X_3 = n$, the infimum is $0$. 
\end{enumerate}
\end{enumerate}
\end{proof}

We have the following proposition which proves that the non-existence of a point of minimum for the convexified cost is possible for any number $m$ of classifiers. 

\begin{proposition}
  \label{prop:equivalence-inf}
Non existence of a unique point of minimum:
  \begin{enumerate}[label=(\roman*)]
    \item We have the implication: "if, in a given list $G_1,\cdots,G_m$ there is one couple $(G_i,G_j)$ for which at least one column of its truth table is empty, then the convexified cost for the $m$ classifiers does not have a  unique point of minimum (it might be a non unique minimum or an infimum)".
    \item This implication is an equivalence in the case $m=3$.  
  \end{enumerate}
\end{proposition}

For example, if $(G_1, G_2)$ has an infimum, any additional classifier will yield an infimum, if $(G_1, G_2)$ has a unique minimum, it is enough to test $(G_1, G_3)$ and $(G_2, G_3)$.

The proof of this proposition is postponed to section \ref{prop-proof:equivalence-inf}.

\section{Frontiers}
\label{sec:zoll}

Now we introduce the setup needed for studying the resulting classifier in an abstract way, in particular to control the sign of a given $X_k$, hence the stability of a given $\balpha$.

\subsection{Definition of a frontier and characterization of the contradictions}

The curves $\{ X_k(\balpha) =0\}$ separate the decisions $X_k>0$ and $X_k<0$ and they are of interest, namely identifying regions which return different decisions for two given loss functions.  The are called $\phi$-frontiers, and depend only on $m$, $\balpha$ and $\phi$. 

\begin{definition}[$\phi$-frontier points]
A $\phi$-frontier point $\balpha$ is a regular point for $\phi$ with at least one of the values of $X_{\phi,k}^{min}(\balpha)$ equals to zero.

For $m \geq 3$, a corner point of order $p\geq 2$ for $\phi$ is a set of relations between  the coordinates of $\balpha$ such that exactly $p$ relations $X_{\phi,k_1}=0,\ldots, X_{\phi,k_p}=0$ hold simultaneously.
  \label{def:frontiers}
\end{definition}

As we have a complete study for two particular loss functions of the point of minimum of the convexified risk for $3$ classifiers, we have analytical equations for the frontiers and a study of contradictory decisions. This is described below for $\balpha \in {\cal A}^3$. 

Let $a,b,c,d$ be given by
$a =\alpha_0-\alpha_7$,  
$b =\alpha_1-\alpha_6$,  
$c=\alpha_2-\alpha_5$, 
$d=\alpha_3-\alpha_4.$

Recall that $\Gamma_l$ and $\Gamma_b$ are defined in \eqref{eq:Gammal} and \eqref{eq:Gammab}. We have an explicit expression of the frontiers in the cases of logit and boost:

\begin{proposition}
  \label{equations-frontiers:lb} 

  The equations of frontiers are, for logit, resp. for boost:

  \begin{itemize}
    \item $\Gamma_l(-a/2,\balpha) = 0$, resp. $\Gamma_b(a,\balpha) = \alpha_0 \alpha_3 \alpha_5 \alpha_6$ for $X_0=0$,
    \item $\Gamma_l(b/2,\balpha) = 0$, resp. $\Gamma_b(-b,\balpha) = \alpha_0 \alpha_3 \alpha_5 \alpha_6$ for $X_1=0$,
    \item $\Gamma_l(c/2,\balpha) = 0$, resp. $\Gamma_b(-c,\balpha) = \alpha_0 \alpha_3 \alpha_5 \alpha_6$ for $X_2=0$,
    \item $\Gamma_l(-d/2,\balpha) = 0$, resp. $\Gamma_b(d,\balpha) = \alpha_0 \alpha_3 \alpha_5 \alpha_6$ for $X_3=0$.
  \end{itemize}
\end{proposition}

See proof in appendix \ref{proof:equations-minimum-frontiers:lb}.

We have thus a complete description of the frontiers and a procedure to identify the minima by solving (either numerically for the boost loss function through a 1D Newton method or through Cardan formulae for logit).
The choice of the logit function which is a regularized strictly convex version of the ReLU function could be profitable thanks to this.

\subsection{The frontiers as a tool for understanding infima of the convexified cost}

The traditional approach to study the cases of infima is by regularizing the strictly convex function. It has been shown in appendix \ref{sec:logical-2} that all regularizations are equivalent.

Authors then generally consider the limit direction of $\bbeta$. 
In this subsection, we show in two examples that it is closely related to the notion of {\it frontiers}. 

\paragraph{First example: $3$ classifiers with $2$ zero coefficients}
In the case of the convexified cost of remark \ref{rem:inf6}, we use the unique point of minimum of
$\alpha_1\phi(X_1)+\alpha_6\phi(-X_1)+\alpha_3\phi(X_3)+\alpha_4\phi(-X_3)+\alpha_7\phi(-X_0)+\epsilon \phi(X_0)+\alpha_5\phi(X_1-X_0-X_3)+\epsilon \phi(-X_1+X_0+X_3)$.

The Euler equations in $X_1, X_3, X_0$ are

$$\bs\ba{l}
\alpha_1\phi'(X_1)-\alpha_6\phi'(X_1)+\alpha_5\phi'(X_1-X_0-X_3)=\epsilon \phi'(-X_1-X_0+X_3)\\
\alpha_3\phi'(X_3)-\alpha_4\phi'(-X_3)-\alpha_5\phi(X_1-X_0-X_3)=-\epsilon \phi'(-X_1+X_0+X_3)\\
-\alpha_7\phi'(-X_0)-\alpha_5\phi'(X_1-X_0-X_3) + \epsilon \phi'(X_0)= -\epsilon \phi'(-X_1+X_0+X_3)
\ea\es$$

The limit when $\epsilon\rightarrow 0$ of $-X_0$ is $+\infty$, which yields the following relations for the limits of $X_1$ and $X_3$:
$$\bs\ba{l}\alpha_1\phi'(X_1)-\alpha_6\phi'(X_1)+0=0\\
\alpha_3\phi'(X_3)-\alpha_4\phi'(-X_3)-0=0\ea\es$$
hence finite values of $X_1$ and $X_3$.
The limit of $(X_0, X_1, X_3)/\vert (X_0, X_1, X_3)\vert $ is then $(1,0,0)$, which corresponds to a frontier (as in the case of two classifiers in appendix \ref{sec:logical-2}).

\paragraph{Second example: $3$ classifiers with $4$ zero coefficients}
For the function

$\alpha_1\phi(X_1)+\alpha_6\phi(-X_1)+\alpha_3\phi(X_3)+\alpha_4\phi(-X_3)+\alpha_7\phi(-X_0)+\alpha_5\phi(X_1-X_0-X_3)$

one easily checks that

$\forall (X_1, X_3), \alpha_1\phi(X_1)+\alpha_6\phi(-X_1)+\alpha_3\phi(X_3)+\alpha_4\phi(-X_3)\geq \mbox{min }\alpha_1\phi(X_1)+\alpha_6\phi(-X_1)+\alpha_3\phi(X_3)+\alpha_4\phi(-X_3):=min=(\alpha_1+\alpha_6)H(\frac{\alpha_1}{\alpha_1+\alpha_6})+ (\alpha_3+\alpha_4)H(\frac{\alpha_3}{\alpha_3+\alpha_4})$ (with equality at a point $(X_1^*, X_3^*)$. One then show that, for $X_0\rightarrow -\infty$, $\alpha_7\phi(-X_0)+\alpha_5\phi(X_1-X_0-X_3)\rightarrow 0$. Hence there exists a sequence which converges to the infimum of this function, equal to $min$.

This proves that $\frac{(X_0, X_1, X_3)}{\sqrt{X_0^2+X_1^2+X_3^3}}$ converges to $(-1,0,0)$, and this defines $\beta_1^{\mbox{lim}}$, which is a frontier point.

For the function $\alpha_0\phi(X_1+X_2+X_4)+\alpha_1\phi(X_1)+\alpha_2\phi(X_2)+\alpha_4\phi(X_4)$, consider a sequence of the form

$X_1=N, X_2=\theta N, X_4=\mu N$. For all $\theta\geq 0, \mu\geq 0$, this sequence makes the function converge to $0$ when $N$ goes to $+\infty$. Then $(X_1, X_2, X_4)/\sqrt{X_1^2+X_2^2+X_4^2}$ converges to $(1, \theta, \mu)$ hence the limit of the correspondding sequence is arbitrary, including $(0,1,0)$ or $(0,0,1)$.

\paragraph{General case for $3$ classifiers}
For the general analysis, we begin to identify the pair fulls, which yield the value of the corresponding coordinates when the regularization is used. We then consider the other coordinates which go to infinity. The study of these cases correspond to the limits
$X_0\rightarrow \infty$,
$X_1\rightarrow \infty$,
$X_2\rightarrow \infty$,
$X_3\rightarrow \infty$,
where one knows directly the limiting normalized value of $X$ ($\delta_{ji}$).

$(X_0, X_1)\rightarrow (\infty, \infty)$,
$(X_0, X_2)\rightarrow (\infty, \infty)$,
$(X_0, X_3)\rightarrow (\infty, \infty)$,
$(X_1, X_2)\rightarrow (\infty, \infty)$,
$(X_1, X_3)\rightarrow (\infty, \infty)$,
$(X_2, X_3)\rightarrow (\infty, \infty)$
where it is trickier and has to be studied separately.

For example
$\alpha_7\phi(-X_0)+\alpha_1\phi(X_1)+ 0\phi(X_0)+0\phi(-X_1)+\alpha_2\phi(X_2)+\alpha_5\phi(-X_2)$

where  $X_2$ is determined and $X_0\rightarrow -\infty, X_1\rightarrow +\infty$, 
and the limiting value of the normalized version is $X^*=(X_0^*, X_1^*, X_2^*)=(r,1,0)$
thanks to

$$X_0^\epsilon= \beta^*(\frac{\alpha_7}{\alpha_7+\epsilon}), X^1=\beta^*(\frac{\epsilon}{\alpha_1+\epsilon})$$
and one has to study the limit $r$ of $\frac{\beta^*(\frac{\alpha_7}{\alpha_7+\epsilon})}{\beta^*(\frac{\epsilon}{\alpha_1+\epsilon})}$ when $\epsilon\rightarrow 0_+$.

\subsection{Study of data of poor quality with two different classification calibrated functions}

This analysis for the frontiers can even identify values of $\balpha$ for which at least one sign of the $X_j$ calculated from the choice of the boost or logit functions disagree. 

We propose to use this idea to identify these cases explicitly thanks to the characterization of the sign of $X_j$ for all $j$.
\begin{definition}[Data of poor quality] 
  One says that the set $\{\cS,\cH_m\}$ is of poor quality if there exists two functions $\phi_1$ and $\phi_2$, regularly classification calibrated and an index $l\in \{0,\cdots,2^m-1 \}$ such that  
  $X_l^{\phi_1}(\balpha^{\cS,\cH_m}) X_l^{\phi_2}(\balpha^{\cS,\cH_m}) < 0$.
  \label{def:poorquality}
\end{definition} 

In the case of $3$ classifiers, we have exact analytic expressions without solving any equation, which will determine if two resulting classifiers are contradictory for at least one sign and for the two functions boost and logit: 

\begin{lemma}
  \label{prop:robust3}
If, for $\balpha = \balpha^{\cS,\cH_3}$, at least one of the four real numbers \\
  $(\Gamma_b(-b, \alpha)- \alpha_0\alpha_3\alpha_5\alpha_6)\Gamma_l(b/2, \alpha)$,
  $(\Gamma_b(a, \alpha)- \alpha_0\alpha_3\alpha_5\alpha_6)\Gamma_l(-a/2, \alpha)$,\\
  $(\Gamma_b(-c, \alpha)- \alpha_0\alpha_3\alpha_5\alpha_6)\Gamma_l(c/2, \alpha)$,
  $(\Gamma_b(d, \alpha)- \alpha_0\alpha_3\alpha_5\alpha_6)\Gamma_l(-d/2, \alpha)$
  is positive, then the set of classified data $\cS,\cH_3$ is of poor quality.
\end{lemma}

Remark that
$\psi(x):=\frac{(\alpha_0+2x)(\alpha_3+2x)(\alpha_5+2x)(\alpha_6+2x)}{(\alpha_7-2x)(\alpha_4-2x)(\alpha_2-2x)(\alpha_1-2x)} = 1 + \frac{\Gamma_l(x)}{(\alpha_7-2x)(\alpha_4-2x)(\alpha_2-2x)(\alpha_1-2x)},$
hence the function $\psi$ is strictly increasing on $(-\frac12\mbox{min}(\alpha_0, \alpha_3, \alpha_5, \alpha_6), \frac12\mbox{min}(\alpha_7, \alpha_4, \alpha_2, \alpha_1))$ (thanks to the value of the logarithmic derivative $\frac{\psi'(x)}{\psi(x)}$ and $\psi(x)>0$ on this interval).

The functions $\Gamma_l$ and $\Gamma_b$ are thus both increasing. For the frontier $X_0=0$, $x=\alpha_0-\alpha_7=a$ and $x(X_0)=\alpha_0 e^{-X_0}-\alpha_7 e^{X_0}$ is strictly decreasing in $X_0$ and $y=\frac 12(\alpha_7-\alpha_0)=-a/2$, $y(X_0)= \frac{\alpha_7e^{X_0}-\alpha_0}{1+e^{X_0}}$ is strictly increasing in $X_0$. 
Hence, $\Gamma_b(x(X_0), \alpha)$ is strictly decreasing in $X_0$ while $\Gamma_l(y(X_0), \alpha)$ is strictly increasing in $X_0$. 
We deduce that $\Gamma_b(x(0), \alpha)=\Gamma_b(a, \alpha)$ and $\Gamma_l(y(0), \alpha)=\Gamma_l(-a/2,\alpha)$ of the same sign yield opposite signs for $X_0^b(\alpha)$ and $X_0^l(\alpha)$. 

The same procedure applies for $X_1,X_2,X_3$. 

This proves proposition \ref{prop:robust3}.

We illustrate this lemma in the figures \ref{fig:robust00}, \ref{fig:robust23}, \ref{fig:robust33}. 
Note that in these figures, no value of $\alpha_2$, $\alpha_5$ above the blue curve can be considered thanks to $\sum \alpha_i = 1$ and $\alpha_i>0$. In all figures, the red curves are the frontiers for boost, the black ones are the ones for logit, the domain between these two curves corresponds to sets of poor 
quality according to definition \ref{def:poorquality}.

\begin{figure}[h]
  \includegraphics[width=.5\linewidth,height=4.5cm]{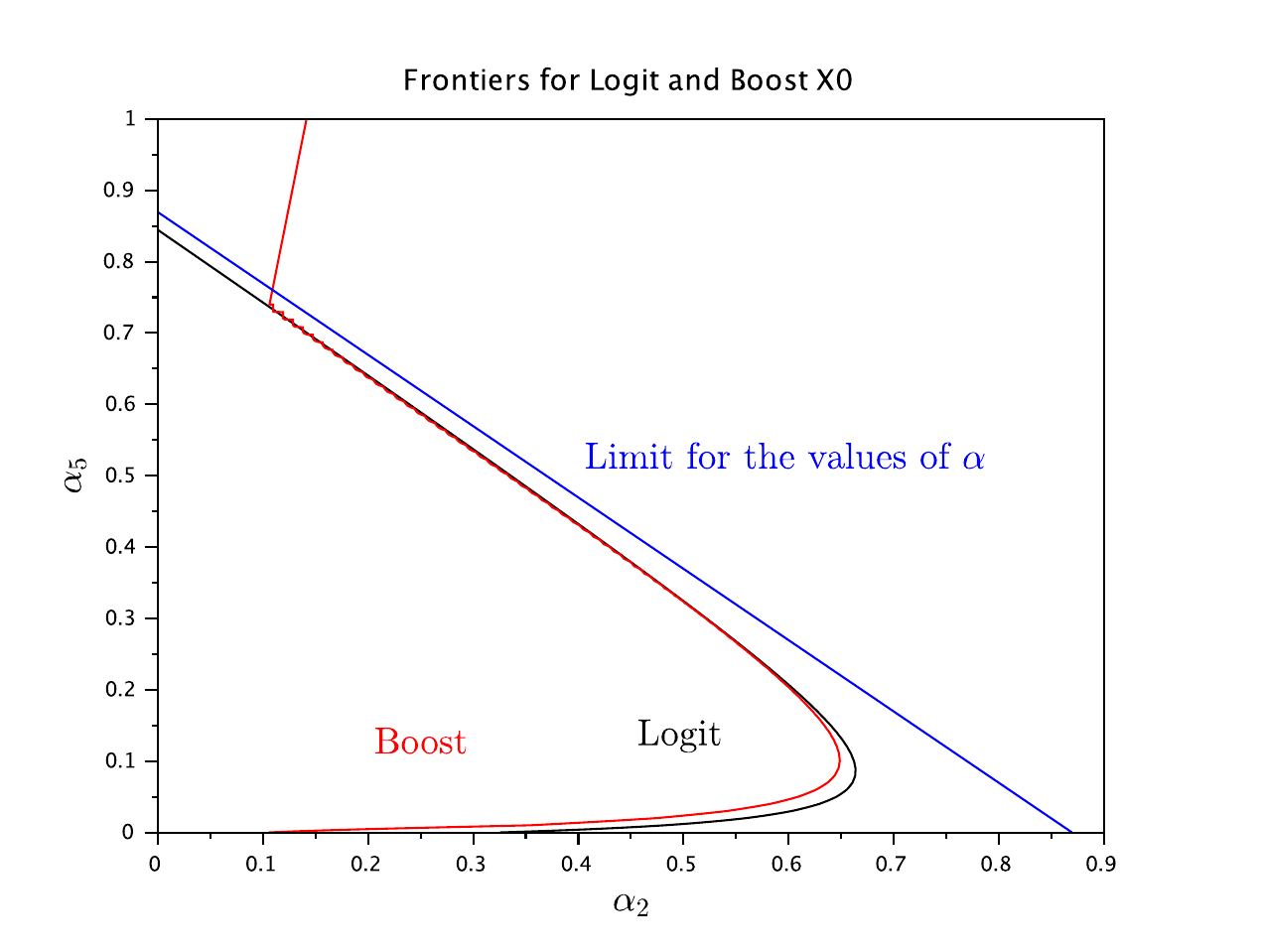}\hfill
  \includegraphics[width=.5\linewidth,height=4.5cm]{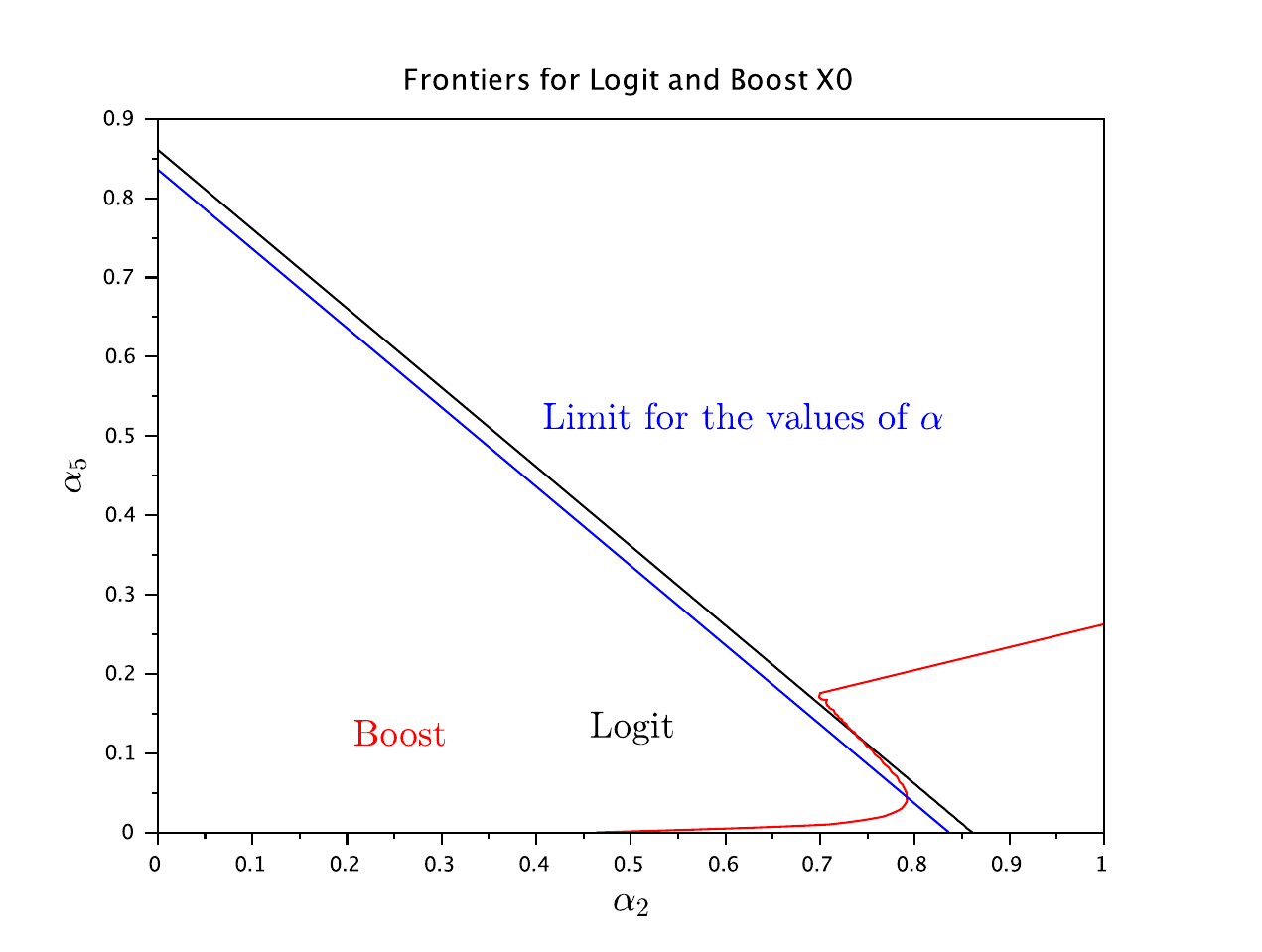}
  \caption{sign of $X_0$ for two differents set of values of $\alpha_0,\alpha_1,\alpha_3,\alpha_4,\alpha_6$.  \label{fig:robust00}} 
\end{figure}

\begin{figure} 
  \includegraphics[width=.5\linewidth,height=4.5cm]{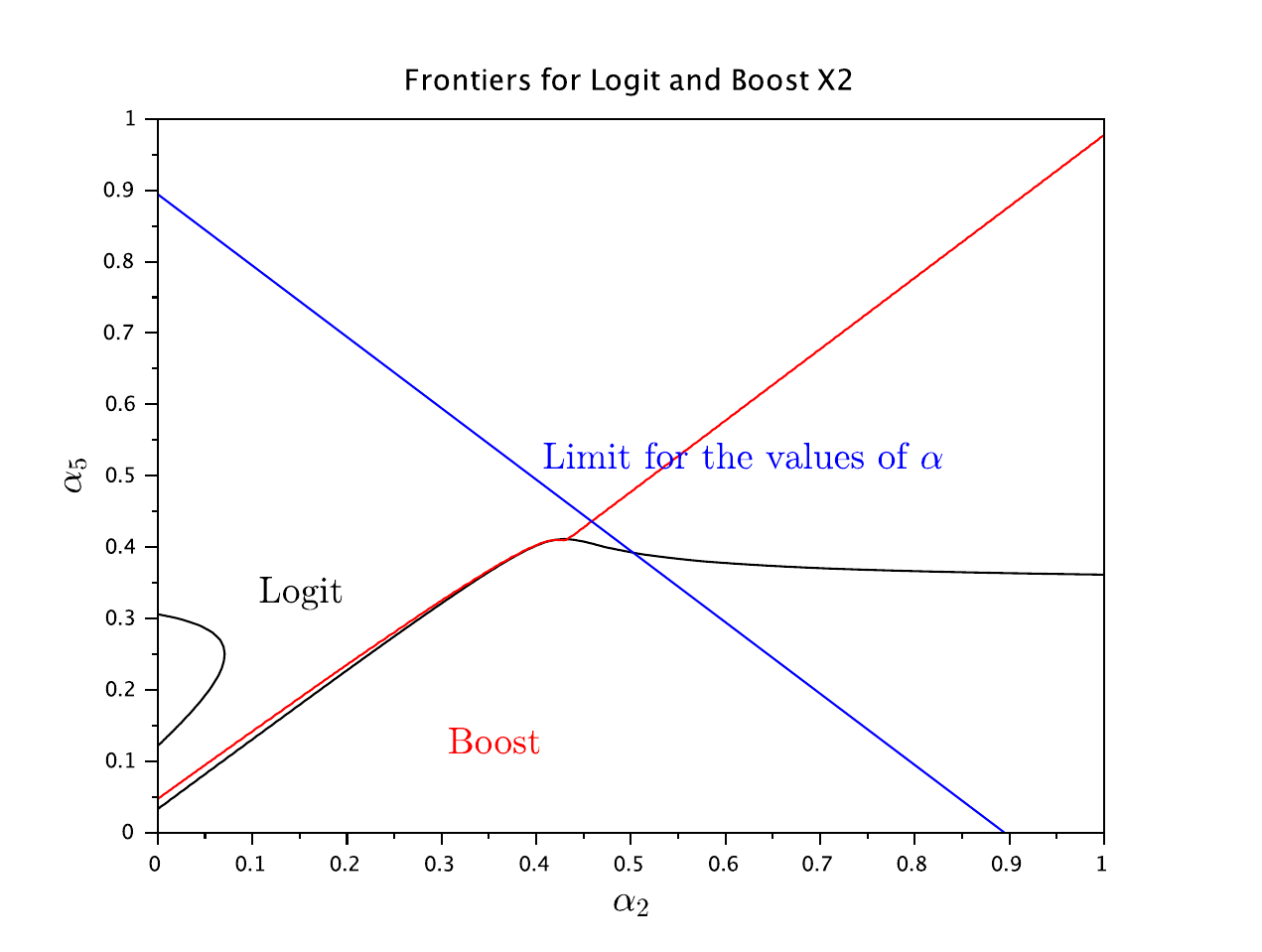}\hfill
  \includegraphics[width=.5\linewidth,height=4.5cm]{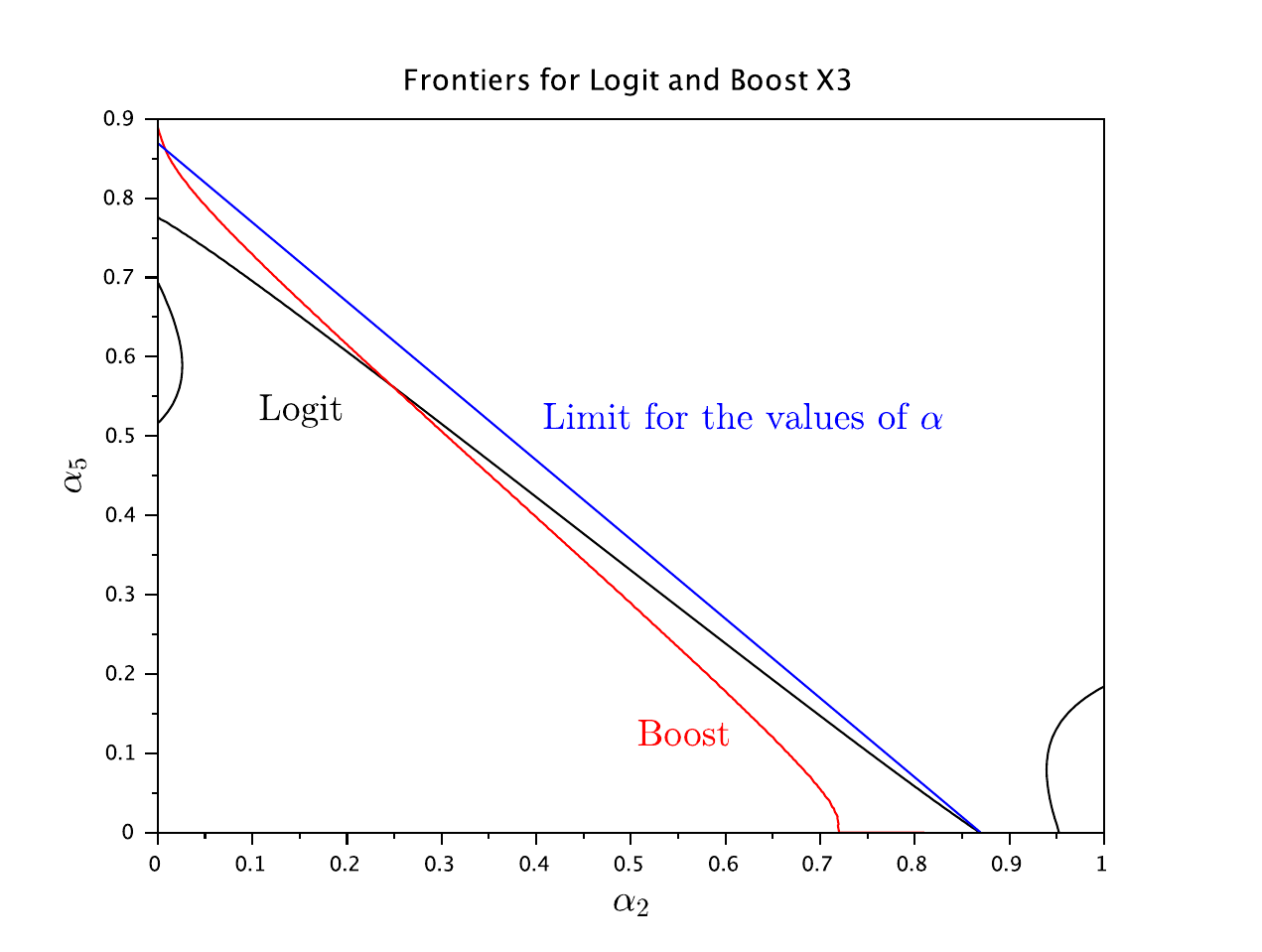}
  \caption{sign of $X_2$ (left) and $X_3$ (right) for the given set.\label{fig:robust23}}
\end{figure}

\begin{figure}
  \includegraphics[width=.5\linewidth,height=4.5cm]{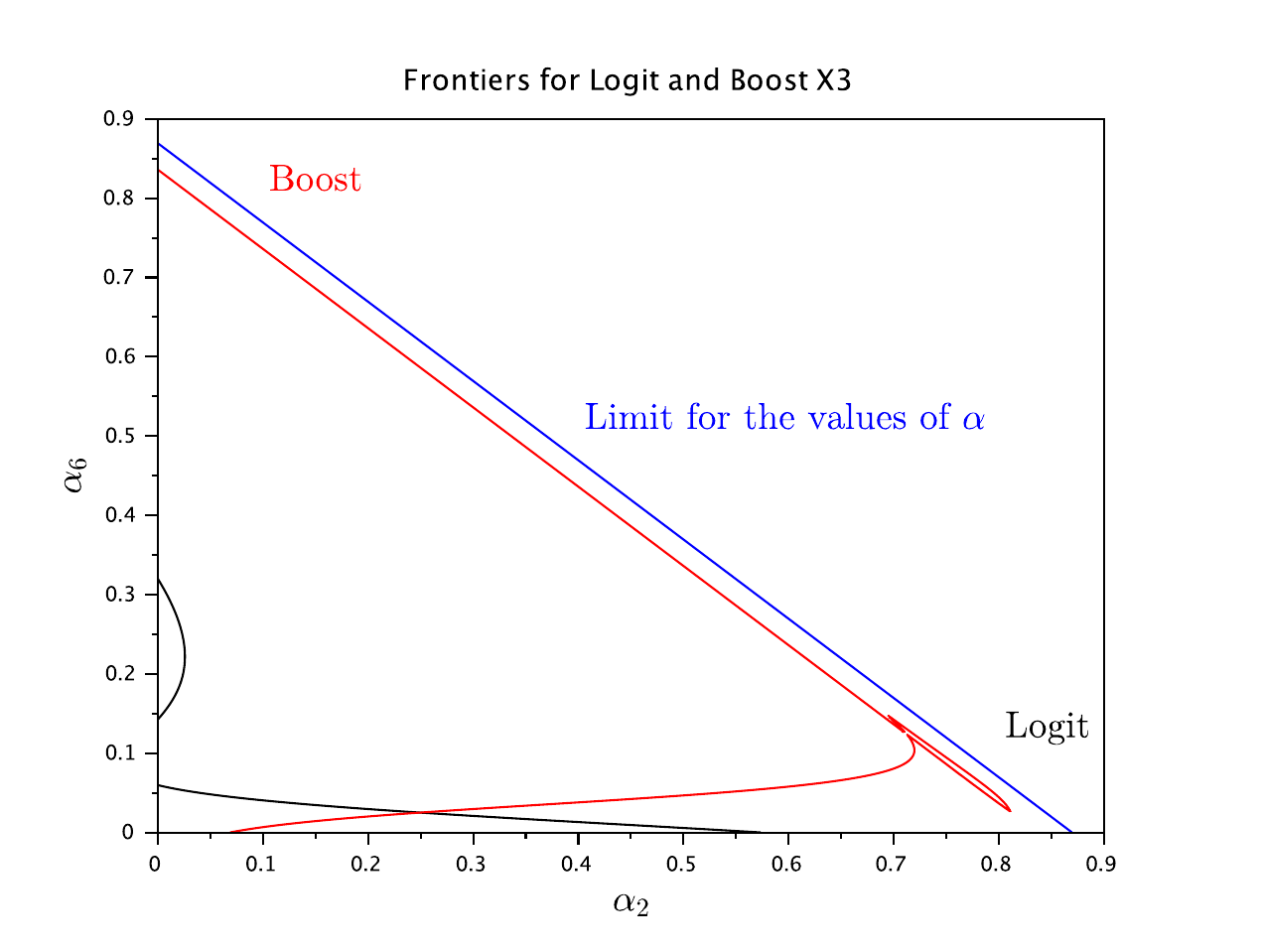}\hfill
  \includegraphics[width=.5\linewidth,height=4.5cm]{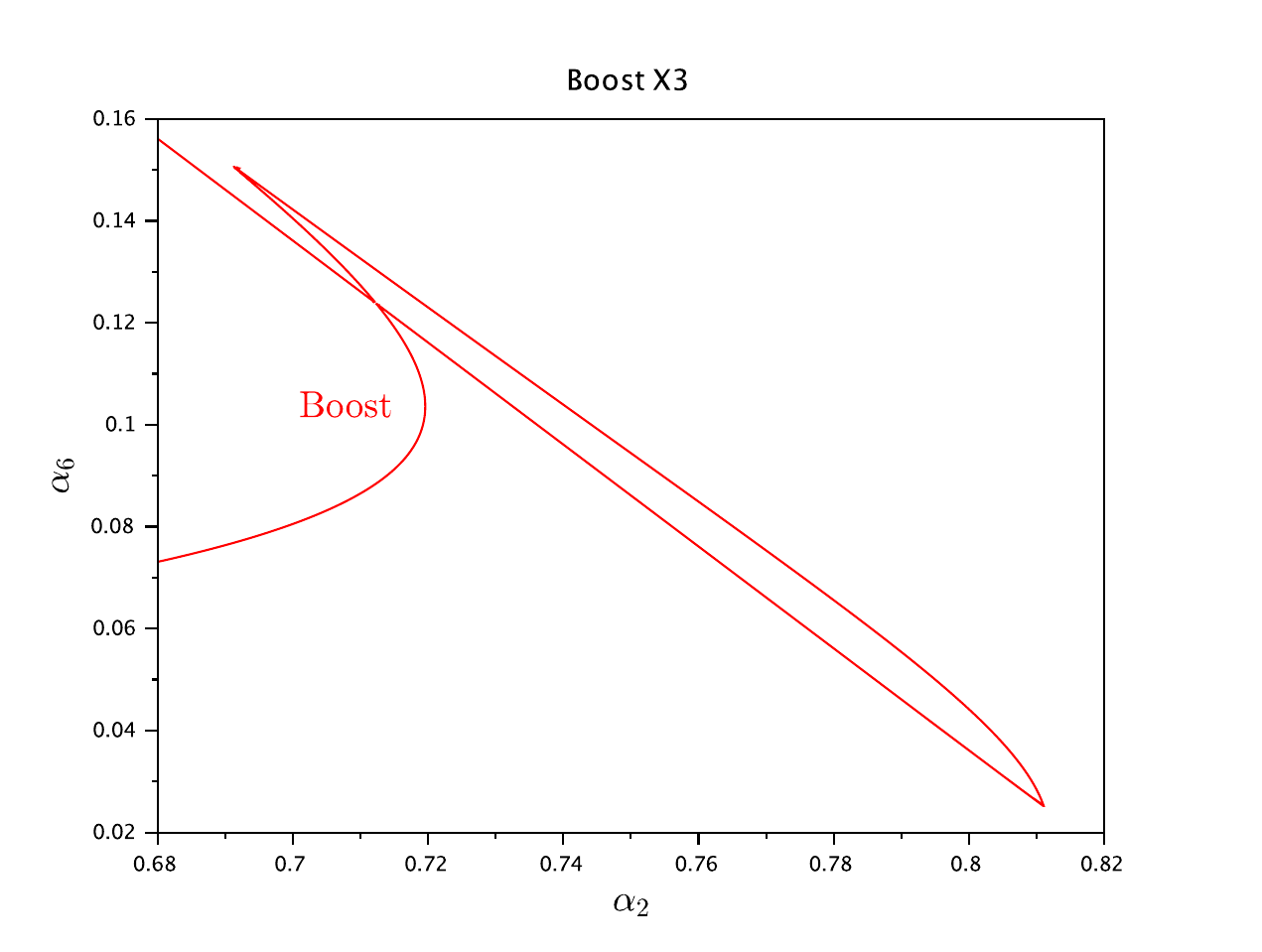}
  \caption{sign of $X_3$ with a zoom (right) near the cusp.\label{fig:robust33}}
\end{figure}

\subsection{A training set of poor quality}

We identify a toy example where the value $\balpha$ corresponds to a region where the quantity $X_j^*$ has not the same sign when using logit and boost, hence not returning the same optimal decision. This example uses a training set, for three given weak classifiers, built to get this contradiction.

\begin{figure}[htbp]
  \centering
  \includegraphics[width=.8\linewidth]{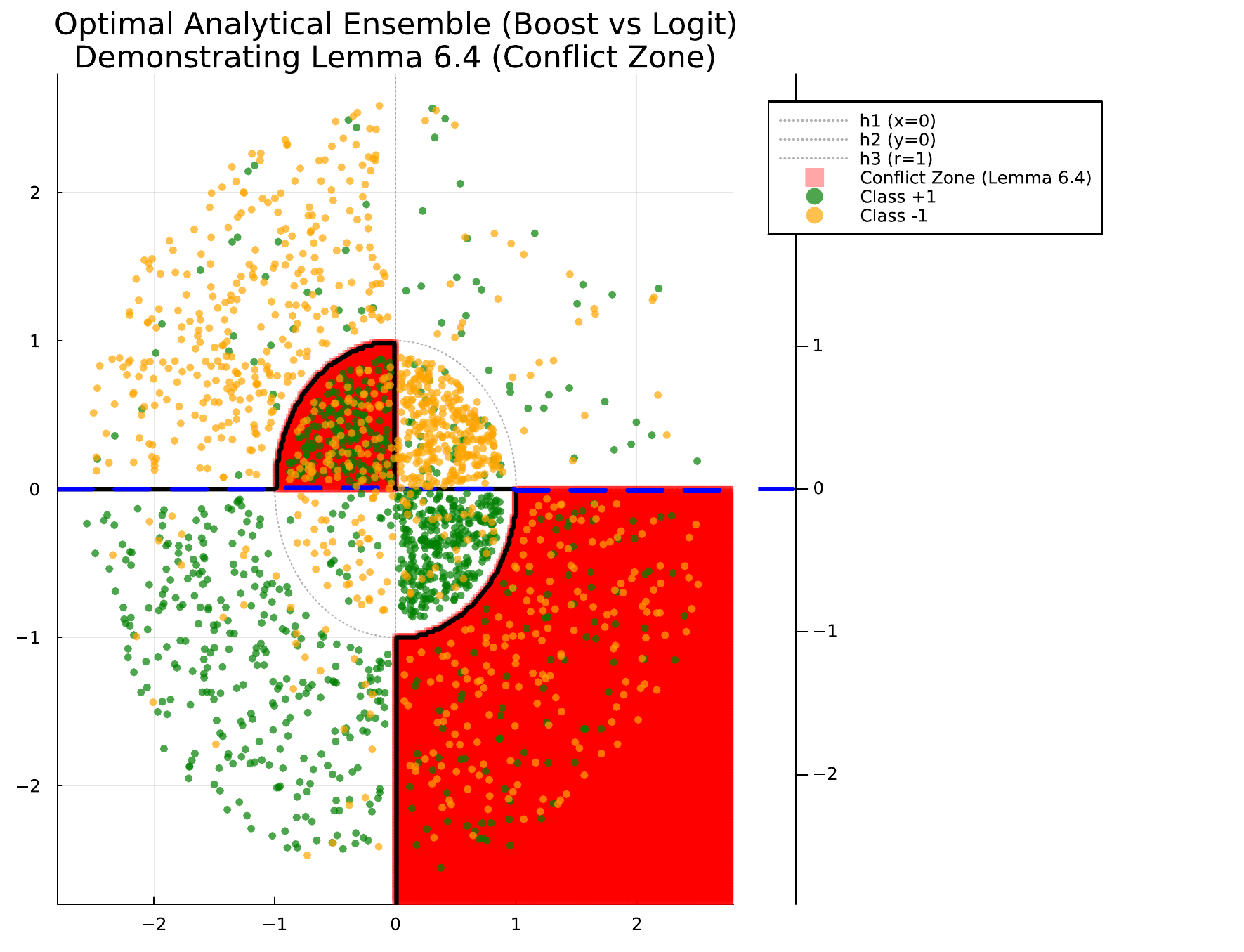}
  \caption{\textbf{Illustration of lemma \ref{prop:robust3} (Data of poor quality).} 
  The two dimensional feature space is partitioned into $8$ logical equivalence classes by three base classifiers ($h_1: x=0$, $h_2: y=0$, and $h_3: x^2+y^2=1$). 
  The optimal ensemble decision boundaries for the logit 
  loss (solid black line) and boost loss (dashed blue line) are plotted using the analytical weights derived in theorem \ref{equations-minimum-frontiers:lb}. The red shaded areas represent the conflict zone where the two risks return opposite decisions. Lemma \ref{prop:robust3} predicts the existence of such conflicts.}
  \label{fig:empirical_conflict}
\end{figure}

\paragraph{Empirical Illustration of lemma \ref{prop:robust3} and $\phi$-frontiers}
Figure \ref{fig:empirical_conflict} provides an illustration of poor quality data (definition \ref{def:poorquality}). In this two dimensional setup, the three base classifiers partition the input space into 8 logical configurations $\bepsilon$ defining the truth table $\balpha$. Within each region, the ensemble's prediction is constant and dictated by the sign of $X_k(\bbeta) = \bepsilon(k) \cdot \bbeta$.

By deriving the optimal weights via theorem \ref{equations-minimum-frontiers:lb}, the resulting decision boundaries for the logit and boost losses are plotted. As observed, the boundaries coincide in most regions but diverge to enclose specific geometric areas highlighted in red. These conflict zones correspond to the logical equivalence classes where the two optimal models output  opposite decisions. For instance, the red regions in figure \ref{fig:empirical_conflict} correspond to the configurations $\bepsilon = (-1, 1, -1)$ and its negation $(1, -1, 1)$, which are governed by the variable $X_2$.

If one of the four products in lemma \ref{prop:robust3} is strictly positive --- such as $(\Gamma_b(-c, \balpha) - \alpha_0\alpha_3\alpha_5\alpha_6)\Gamma_l(c/2, \balpha) > 0$ --- it guarantees that the optimal roots for boost and logit fall on opposite sides of the threshold $X_2=0$. 

In the feature space, this algebraic condition triggers the appearance of the red conflict zone associated with $X_2$. 
This confirms that poor quality data reflects an intrinsic ambiguity in the dataset's logical distribution, forcing standard classification-calibrated losses to yield contradictory predictions.

\section{Conclusions and perspectives}

This paper introduced a novel framework for ensemble learning based on a logical structuration of the dataset via truth tables. This approach partitions the data into equivalence classes defined by the decisions of the elementary classifiers, leading to a  compression of the training set information where the complexity depends on $2^m$ rather than the number of examples $n$. This structure naturally supports hierarchical extensions, such as adding new classifiers.

We generalized the theory of classification calibrated functions to the multidimensional setting. This theoretical framework allowed us to perform a rigorous analysis of the convexified empirical risk:
\begin{itemize}
\item We established sufficient conditions for the existence and uniqueness of a global minimum for an arbitrary number of classifiers.
\item In the specific case of three classifiers, we provided an exhaustive classification of all possible configurations, identifying necessary and sufficient conditions for the existence and uniqueness of the solution, as well as characterizing cases leading to an infimum or non-unique minima.
\item We showed that the non-existence of a minimum is intrinsically linked to the lack of contradictions between classifiers on the training set (i.e., empty columns in the truth table).  
It yields unstability on decisions even when one uses the classical trick of Tikhonov regularization or when we add elements in empty columns.
\end{itemize}

Furthermore, our analysis enabled the derivation of explicit analytical formulae for the optimal classifier weights in the case of three classifiers for the Exponential (boost) and Logistic (logit) loss functions, bypassing the need for iterative numerical optimization in these cases. The analytic result using Cardan formulae is obtained for $m=3$ and the function logit. 

We note that our approach is ineffective for $m$ large because, in this case, the convexified cost is likely to have an infimum (which again leads to unstability of decisions). 

Finally, we introduced the concept of $\phi$-frontiers in the parameter space. These frontiers define the boundaries where the resulting classifier's decision changes, allowing for a precise sensitivity analysis with respect to the choice of the loss function. This led to a criterion for identifying "data of poor quality," where different standard loss functions yield contradictory predictions.

Perspectives include further investigation into the characterization of data quality based on the stability analysis or proving the conjecture that the remark of the case $m=3$ in proposition \ref{prop:equivalence-inf} is true for all $m$.

\bibliography{/Users/brossiej/Documents/Pro/Oeuvres/bib/biblio-ada.bib}

\newpage 

\appendix

\section{Proof of lemma \ref{lem:logical-2}}
\label{sec:logical-2}

\begin{proof}
The sign of the resulting classifier is given by the signs of $\beta_1+\beta_2$ and $\beta_1-\beta_2$. 
The false decision table is given by the sign of $\beta_1 G_1(x_i,y_i) + \beta_2 G_2(x_i,y_i)$ for all $i$.
The four possible configurations are:

\begin{center}
  \begin{tabular}{|c|c|c|c|c|}
    \hline 
  $G_1$   &  $-1$ & $-1$ & $+1$ & $+1$ \tabularnewline
    \hline 
  $G_2$   &  $-1$ & $+1$ & $-1$ & $+1$ \tabularnewline
    \hline 
   $\beta_1 G_1+\beta_2 G_2$  & $-\beta_1-\beta_2$ & $-\beta_1+\beta_2$ & $+\beta_1-\beta_2$ & $+\beta_1+\beta_2$ \tabularnewline
    \hline 
  
    \end{tabular}
\end{center}
hence we have
$$
\begin{array}{ll}
\mbox{Sectors (cf. figure \ref{fig:sectors})} & \mbox{Number of errors of } \beta_1 h_1 + \beta_2 h_2 \\ 
\beta_1 + \beta_2 > 0, \beta_1 - \beta_2 > 0  & \sharp\{G_1\!\!=\!-1, G_2\!\!=\!+1\} + \sharp\{G_1\!\!=\!-1, G_2\!\!=\!-1\}  
= \sharp\{G_1\!\!=\!-1\} \\
\beta_1 + \beta_2 > 0, \beta_1 - \beta_2 < 0  & \sharp\{G_1\!\!=\!+1, G_2\!\!=\!-1\} + \sharp\{G_1\!\!=\!-1, G_2\!\!=\!-1\} 
= \sharp\{G_2\!\!=\!-1\}\\
\beta_1 + \beta_2 < 0, \beta_1 - \beta_2 > 0  & \sharp\{G_1\!\!=\!+1, G_2\!\!=\!+1\} + \sharp\{G_1\!\!=\!-1, G_2\!\!=\!+1\} =  \sharp\{G_2\!\!=\!+1\}\\
\beta_1 + \beta_2 < 0, \beta_1 - \beta_2 < 0  & \sharp\{G_1\!\!=\!+1, G_2\!\!=\!+1\} + \sharp\{G_1\!\!=\!+1, G_2\!\!=\!-1\} = \sharp\{G_1\!\!=\!+1\}\\
\end{array}
$$
\begin{figure}[H]
  \begin{center}
        \begin{tikzpicture}[scale=.9]
            \draw[color=black,dashed,->] (-3,0)--(+4,0) ;
            \draw[color=black,dashed,->] (0,-3)--(0,+3) ;
            \draw (3.8,0.6) node[below] {$\beta_1$} ;
            \draw (0.3,3.) node[below] {$\beta_2$} ;
            \draw (3,2.5) node[below] {$X_1 = 0$} ;
            \draw (3,-1.9) node[below] {$X_0 = 0$} ;
            \draw[color=black] (-3,-3)--(+3,+3) ;
            \draw[color=black] (-3,+3)--(+3,-3) ;

            \draw[color=black] (2.,0) node[below]{$X_0 < 0$} ;
            \draw[color=black] (2.,.7) node[below]{$X_1 < 0$};
            \draw[color=black] (2,1.2) node[below]{Line 1}; 

            \draw[color=black] (-1.9,.7) node[below]{$X_1 > 0$} ;
            \draw[color=black] (-1.9,0) node[below]{$X_0 > 0$} ;
            \draw[color=black] (-1.9,1.2) node[below]{Line 4};

            \draw[color=black] (0.1,2.) node[below]{$X_1 > 0$} ;
            \draw[color=black] (0.1,1.6) node[below]{$X_0 < 0$} ;
            \draw[color=black] (-0.3,2.4) node[below]{Line 2} ;

            \draw[color=black] (0,-1.6) node[below]{$X_1 < 0$} ;
            \draw[color=black] (0,-2.) node[below]{$X_0 > 0$} ;
            \draw[color=black] (-0.3,-1.2) node[below]{Line 3} ;
        \end{tikzpicture}
  \caption{Sectors and frontiers for $m=2$. $X_0=-\beta_1-\beta_2$, $X_1=-\beta_1+\beta_2$. Each sector corresponds to a line in the previous table. \label{fig:sectors}}
  \end{center}
\end{figure}
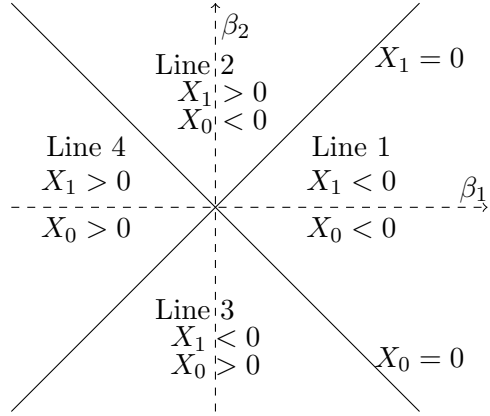
Consider $\beta_1, \beta_2$ such that the number of errors of the classifier $\mbox{sign}(\beta_1 h_1 + \beta_2 h_2)$ is the smallest. 
According to the exhaustive choices of the table in figure \ref{fig:sectors}, the smallest number of errors corresponds to one of the situations : smallest number of errors of $G_1$, smallest number of errors of $G_2$, smallest number of errors of NOT$(G_1)$, smallest number of errors of NOT$(G_2)$.

The cases of equalities correspond to choosing indifferently one of the classifiers which present the smallest number of errors. 
\end{proof}

The value of the unique point of minimum $\bbeta_\phi^{min}(\balpha^{\cS,\cH_2})$ in lemma \ref{lemma:eu} is given by 

\begin{small}
      $$-\left( 
        \frac 12 (\beta^\star (\frac{\alpha_0^{\cS,\cH_2}}{\alpha_0^{\cS,\cH_2}+\alpha_3^{\cS,\cH_2}} ) +  \beta^\star (\frac{\alpha_1^{\cS,\cH_2}}{\alpha_1^{\cS,\cH_2}+\alpha_2^{\cS,\cH_2}} )),   
        \frac 12 (\beta^\star (\frac{\alpha_0^{\cS,\cH_2}}{\alpha_0^{\cS,\cH_2}+\alpha_3^{\cS,\cH_2}} ) -  \beta^\star (\frac{\alpha_1^{\cS,\cH_2}}{\alpha_1^{\cS,\cH_2}+\alpha_2^{\cS,\cH_2}} ) ) \right)$$
\end{small}

 using $\beta^\star$ defined in section \ref{sec:generalization-of-classification-calibrated-functions}.

\section{Proof of theorem \ref{theo:zoology-europe}}
\label{sec:proof-theo:zoology-europe}

\paragraph*{Proof of (i)}

This is a direct consequence of lemma \ref{lemma:exist-unique} for the case $m=3$.\\

\paragraph*{Proof of (ii)}

For the case of three non-zero coefficients, the function is $\alpha_i\phi(X_i)+\alpha_j\phi(X_j)+\alpha_k\phi(X_k)$.
  Either $X_i,X_j,X_k$ forms a system of coordinates in $\R^3$, hence an infimum when $X_i,X_j,X_k$ tends to $+\infty$ or 
  $(i,j)$ is a pair (reordering indexes if needed) and it is enough to let $X_k$ go to $+\infty$. 
  
For the case of only one non-zero coefficient, we get an infimum. 

For the case of two non-zero coefficients, we get an infimum if they do not belong to the same pair; otherwise we get a non-unique minimum.\\

  \paragraph*{Proof of (iv)}
We treat successively the case of $6$ and $5$ non-zero elements:
\begin{enumerate}[label=(\roman*)]
\item Consider the case of only $6$ non-zero elements. Assume in addition   that there are only $2$ full pairs (the case of $3$ full pairs is addressed in (i) of theorem \ref{theo:zoology-europe}). 
In this case, the functions listed below have a unique point of minimum and are the only ones in this case.

We present this list by ordering the terms in the function beginning with the $2$ full pairs.

The non full pairs, characterized for example by the unknown $X_j$, $j=0\ldots 3$, contain $X_j$ and $-X_j$ in the arguments of the two last terms using the relation $X_0+X_3=X_1+X_2$: $j=2$ for example for the first function of the list below as it is written. The two last terms could also be written 
$\alpha_2\phi(X_0+X_3-X_1)+\alpha_4\phi(-X_3)$ and $j=3$. In the following array, the first column lists the full pairs considered.
\begin{equation*}
  \label{liste1}
  \begin{scriptsize}
  \begin{array}{lll}
   \text{Pair } (0,1) & \alpha_0\phi(X_0)+\alpha_7\phi(-X_0)+\alpha_1\phi(X_1)+\alpha_6\phi(-X_1)&+\alpha_2\phi(X_2)+\alpha_4\phi(X_0-X_1-X_2)\\
    &\alpha_0\phi(X_0)+\alpha_7\phi(-X_0)+\alpha_1\phi(X_1)+\alpha_6\phi(-X_1)&+\alpha_5\phi(-X_2)+\alpha_3\phi(-X_0+X_1+X_2)\\
    \text{Pair } (0,2) &\alpha_0\phi(X_0)+\alpha_7\phi(-X_0)+\alpha_2\phi(X_2)+\alpha_5\phi(-X_2)&+\alpha_1\phi(X_1)+\alpha_4\phi(X_0-X_1-X_2)\\
    &\alpha_0\phi(X_0)+\alpha_7\phi(-X_0)+\alpha_2\phi(X_2)+\alpha_5\phi(-X_2)&+\alpha_6\phi(-X_1)+\alpha_3\phi(-X_0+X_1+X_2)\\
    \text{Pair } (0,3) & \alpha_0\phi(X_0)+\alpha_7\phi(-X_0)+\alpha_3\phi(X_3)+\alpha_4\phi(-X_3)&+\alpha_1\phi(X_1)+\alpha_2\phi(X_0+X_3-X_1)\\
    &\alpha_0\phi(X_0)+\alpha_7\phi(-X_0)+\alpha_3\phi(X_3)+\alpha_4\phi(-X_3)&+\alpha_6\phi(-X_1)+\alpha_5\phi(-X_0+X_1-X_3)\\
    \text{Pair } (1,2) & \alpha_1\phi(X_1)+\alpha_6\phi(-X_1)+\alpha_2\phi(X_2)+\alpha_5\phi(-X_2)&+\alpha_0\phi(X_0)+\alpha_3\phi(-X_0+X_1+X_2)\\
    & \alpha_1\phi(X_1)+\alpha_6\phi(-X_1)+\alpha_2\phi(X_2)+\alpha_5\phi(-X_2)&+\alpha_7\phi(-X_0)+\alpha_4\phi(X_0-X_1-X_2)\\
    \text{Pair } (1,3) & \alpha_1\phi(X_1)+\alpha_6\phi(-X_1)+\alpha_3\phi(X_3)+\alpha_4\phi(-X_3)&+\alpha_0\phi(X_0)+\alpha_5\phi(X_1-X_0-X_3)\\
    & \alpha_1\phi(X_1)+\alpha_6\phi(-X_1)+\alpha_3\phi(X_3)+\alpha_4\phi(-X_3)&+\alpha_7\phi(-X_0)+\alpha_2\phi(-X_1+X_0+X_3)\\
    \text{Pair } (2,3) & \alpha_2\phi(X_2)+\alpha_5\phi(-X_2)+\alpha_3\phi(X_3)+\alpha_4\phi(-X_3)&+\alpha_0\phi(X_0)+\alpha_6\phi(-X_0-X_3+X_2)\\
    &\alpha_2\phi(X_2)+\alpha_5\phi(-X_2)+\alpha_3\phi(X_3)+\alpha_4\phi(-X_3)&+\alpha_7\phi(-X_0)+\alpha_1\phi(X_0+X_3-X_2)
  \end{array}
  \end{scriptsize}
\end{equation*}
When, in this array, one changes exactly one element of one of the two last pairs, one has an infimum.

We treat only two examples, one in this list where there is a solution of the Euler equations (and hence a unique point of minimum as the function is strictly convex) and one where there is an infimum.

$\bullet$ An example where there is an infimum is 

\begin{scriptsize}
  $C_{\boldsymbol{\alpha}}(\bbeta)= \alpha_1\phi(X_1)+\alpha_6\phi(-X_1)+\alpha_3\phi(X_3)+\alpha_4\phi(-X_3)+\alpha_7\phi(-X_0)+\alpha_5\phi(X_1-X_0-X_3)$.
\end{scriptsize} 
Indeed, the derivative of this function with respect fo $X_0$ writes $-\alpha_7\phi'(-X_0)-\alpha_5\phi'(X_1-X_0-X_3) > 0$, hence the Euler equations have no solution, hence no point of minimum.

$\bullet$  An example where there is a minimum is 

\begin{scriptsize}
$C_{\boldsymbol{\alpha}}(\bbeta)= \alpha_1\phi(X_1)+\alpha_6\phi(-X_1)+\alpha_3\phi(X_3)+\alpha_4\phi(-X_3)+\alpha_0\phi(X_0)+\alpha_5\phi(X_1-X_0-X_3)$.
\end{scriptsize}

The Euler equations write
$\alpha_1\phi'(X_1)-\alpha_6\phi'(-X_1)+\alpha_5\phi'(X_1-X_0-X_3)=0,
\alpha_3\phi'(X_3)-\alpha_4\phi'(-X_3)-\alpha_5\phi'(X_1-X_0-X_3)=0,
\alpha_0\phi'(X_0)-\alpha_5\phi'(X_1-X_0-X_3)=0.$

Let $a= \alpha_5\phi'(X_1-X_0-X_3)$. This system rewrites
$$\begin{array}{l}\alpha_1\phi'(X_1)-\alpha_6\phi'(-X_1)=-a,\\
\alpha_3\phi'(X_3)-\alpha_4\phi'(-X_3)=a,\\
\alpha_0\phi'(X_0)=a. \end{array}$$

Denote by $l$ (possibly $-\infty$) the limit of $\phi'$ at $-\infty$ and recall that the limit of $\phi'$ at $+\infty$ is zero. A necessary condition on $a$ for the system on $(X_0, X_1, X_3)$ to have a solution is simultaneously to have $-a\in (\alpha_1l, -\alpha_6l)$, $a\in (\alpha_3l, -\alpha_4l)$, $a\in (\alpha_0l, 0)$, that is 
$$a\in (\mbox{min}(\alpha_0, \alpha_3, \alpha_6)l, 0).$$
Moreover, the limit, when $a$ goes to 0, of $X_0$ is $-\infty$, the limit of $X_1$ and $X_3$ are finite numbers. Similarly, the limit when $a\rightarrow \mbox{min}(\alpha_0, \alpha_3, \alpha_6)l$ of at least one of the values of $X_0, -X_1, X_3$ is $-\infty$, that is $X_1-X_0-X_3$ go to $+\infty$.
One deduces $X_0= (\phi')^{-1}(\frac{a}{\alpha_0})$, $X_1= B_{\alpha_1, \alpha_6}(-a)$, $X_3= B_{\alpha_3, \alpha_4}(a)$, where $B_{\theta, \tau}$ is the reciprocal function of the strictly increasing function $x\rightarrow \theta\phi'(x)-\tau\phi'(-x)$. The function $B_{\theta, \tau}$ is thus strictly increasing, and the equation for finding $a$ is
\begin{equation}
  \label{eq:sol11}
  \alpha_5\phi'(B_{\alpha_1, \alpha_6}(-a)- (\phi')^{-1}(\frac{a}{\alpha_0})-B_{\alpha_3, \alpha_4}(a))-a=0.
\end{equation}

The function in the left-hand side of \eqref{eq:sol11} from  
$(\mbox{min}(\alpha_0, \alpha_3, \alpha_6)l, 0)$ onto $\R$ is strictly decreasing  hence it has a unique root.

\item Consider the case of five non zero elements.
\begin{itemize}
\item if two pairs $(i,j)$ are full, one additional element of the form $\alpha_k\phi(X_k)$ one deduces $X_i$ and $X_j$, and $X_k\rightarrow +\infty$ yields an infimum. 
\item if one pair $i$ is full, the three other pairs contain, each, only one element. We have cases of minimum and cases of infimum
\begin{description}
\item[Case of minimum.]

Let us treat the first one. The function one minimizes is
$\alpha_0\phi(X_0)+\alpha_7\phi(-X_0)+\alpha_1\phi(X_1)+\alpha_2\phi(X_2)+\alpha_4\phi(X_0-X_1-X_2)$.

Euler system of equations  is
$\alpha_0\phi'(X_0)-\alpha_7\phi'(-X_0)+\alpha_4\phi'(X_0-X_1-X_2)=0, 
\alpha_1\phi'(X_1)-\alpha_4\phi'(X_0-X_1-X_2)=0, 
\alpha_2\phi'(X_2)-\alpha_4\phi'(X_0-X_1-X_2)=0.$
Introduce $b= \alpha_4\phi'(X_0-X_1-X_2)$. One obtains $X_1= (\phi')^{-1}(\frac{b}{\alpha_1})$, $X_2= (\phi')^{-1}(\frac{b}{\alpha_2})$, $X_0= B_{\alpha_0, \alpha_7}(-b)$, which yields 
\begin{equation}
  \label{eq:sol13}
\alpha_4\phi'(B_{\alpha_0, \alpha_7}(-b)-(\phi')^{-1}(\frac{b}{\alpha_1})- (\phi')^{-1}(\frac{b}{\alpha_2}))-b=0.
\end{equation}

The function in the left-hand side of \eqref{eq:sol13},  from $(\alpha_7l, 0)$ onto $\R$, is strictly decreasing hence has a unique root.

As the 3-dimensional generic convexified cost is strictly convex, the solution is a point of minimum, hence the function has a unique minimum.\\

\item[Case of infimum.] When one of the coefficients changes, there is an infimum. For example, consider 

$(X_0, X_1, X_2)\rightarrow \alpha_0\phi(X_0)+\alpha_7\phi(-X_0)+\alpha_1\phi(X_1)+\alpha_5\phi(-X_2)+\alpha_4\phi(X_0-X_1-X_2)$.
The derivative with respect to $X_2$ is $-\alpha_5\phi'(-X_2)-\alpha_4\phi'(X_0-X_1-X_2) > 0$ hence no solution to the Euler equations, there is an infimum.
\end{description}
\end{itemize}

By inspection, we observe that ALL the cases of unique point of minimum described above contain at least either the indexes $(0,3,5,6)$ or $(1,2,4,7)$.
This ends the proof of (iv) of theorem \ref{theo:zoology-europe}.

\begin{flushright}
$\square$ 
\end{flushright}
\end{enumerate}

\section{Proof of proposition \ref{prop:equivalence-inf}}
\label{prop-proof:equivalence-inf}

\begin{enumerate}
  \item 
  For the proof of the first item of proposition \ref{prop:equivalence-inf}, assume we have a list of $m$ classifiers and that the couple $(G_i,G_j)$ induces at least one column of its truth table empty.
  
  $\bullet$ If only one column is empty, without loss of generality, we assume that it is $G_i=-1, G_j=+1$, hence 
  $\sharp\{G_i=+1, G_j=+1\} > 0$, 
  $\sharp\{G_i=-1, G_j=-1\} > 0$, 
  $\sharp\{G_i=+1, G_j=-1\} > 0$ and 
  $\sharp\{G_i=-1, G_j=+1\} = 0$.
  
  Define the following subsets of $\cS$:
  
  $\cS^{i+,j+} = \{ (x_p,y_p): G_i(x_p,y_p)=+1, G_j(x_p,y_p)=+1 \}$, 
  
  $\cS^{i-,j-} = \{ (x_p,y_p): G_i(x_p,y_p)=-1, G_j(x_p,y_p)=-1 \}$,
  
  $\cS^{i+,j-} = \{ (x_p,y_p): G_i(x_p,y_p)=+1, G_j(x_p,y_p)=-1 \}$ 
  
  and the list of classifiers 
  $\cH_{m-2}^{i,j} = \cH_{m} \setminus \{h_i,h_j\}$.
  Denote by $X=\beta_i+\beta_j$ and $Y=\beta_i-\beta_j$.
  The multidimensional convexified $\phi$-risk writes 
  $$
  \sum_{k=0}^{2^{m-2}-1} 
  \alpha_k^{\cS^{i+,j+}, \cH_{m-2}^{i,j}} \phi(X+Y_k)
  +\alpha_k^{\cS^{i-,j-}, \cH_{m-2}^{i,j}} \phi(-X+Y_k)
  +\alpha_k^{\cS^{i+,j-}, \cH_{m-2}^{i,j}} \phi(Y+Y_k)
  $$
  where $Y_k = \sum_{l=1, l\not= i,j}^{m} \bepsilon(k)_l \beta_l$.
  
  The derivative with respect to $Y$ of this function returns 
  $$
  \sum_{k=0}^{2^{m-2}-1} 
  \alpha_k^{\cS^{i+,j-}, \cH_{m-2}^{i,j}} \phi'(Y+Y_k) < 0
  $$
  Hence, there is no solution to the Euler equations for the convexified $m-\phi$-risk, 
  hence no point of minimum.
  
  $\bullet$ Assume that two columns of the truth table of $G_i,G_j$ are empty.
  Two cases appear: the case of an infimum equals to zero and the case of a non-unique minimum. 
  
  For these two classifiers, the case of non unique minimum is obtained when $\cS^{i+,j-}=\emptyset$, the multidimensional convexified $\phi$-risk is then
  $$
  \sum_{k=0}^{2^{m-2}-1} 
  \alpha_k^{\cS^{i+,j+}, \cH_{m-2}^{i,j}} \phi(X+Y_k)
  +\alpha_k^{\cS^{i-,j-}, \cH_{m-2}^{i,j}} \phi(-X+Y_k)
  $$
  which depends only on $m-1$ independent variables (it does not depend on $Y$), hence even if there is a point of minimum, adding any value of $Y$ also returns a point of minimum. 
  
  Existence and uniqueness of a point of minimum of the 3-dimensional generic convexified cost function  is equivalent to having all the couples 
  $(G_1, G_2)$, $(G_1, G_3)$, $(G_2, G_3)$ leading to a unique point of minimum of the 2-dimensional generic convexified cost function.
  
\item
The proof of the second item of proposition \ref{prop:equivalence-inf} (Equivalence of "no unique minimum for $3$ classifiers" and "one pair of these classifiers at least corresponds to a case of non-uniqueness") consists in the study of all cases of non-uniqueness or existence of the point of minimum for $3$ classifiers.

\begin{enumerate}
\item $6$ non zero elements and $2$ full pairs only. We describe two examples in detail, and all other cases in these two items are listed below  each item without proof.

For example, pick one truth table where the convexified $\phi-$risk has an infimum, the truth table of $G_2, G_3$ leads an infimum too: indeed

$
\begin{array}{ccccccc}
  & \alpha_0 & \alpha_7 & \alpha_3 & \alpha_4 & \alpha_5 & \alpha_1 \\
G_1   \hspace{-.15cm} & -1  & +1  & -1  & +1   & +1   & -1 \\
G_2  \hspace{-.15cm}  & -1  & +1  & +1  & -1   & -1  & -1 \\
G_3  \hspace{-.15cm}  & -1  & +1  & +1  & -1   & +1  & +1 \\
\end{array} \rightarrow
$
$
\begin{array}{ccccc}
  & \alpha_0+\alpha_4 & \alpha_7+\alpha_3 & \alpha_5+ \alpha_1 & 0\\
G_2  \hspace{-.15cm}  & -1  & +1   & -1 & +1\\
G_3  \hspace{-.15cm}  & -1  & +1  & +1 & -1\\
\end{array}
$

and we apply lemma \ref{lemma:eu}.

\item five non zero elements.
\begin{itemize}
\item if two pairs $(i,j)$ are full, one additional element of the form $\alpha_k\phi(X_k)$ one deduces $X_i$ and $X_j$, and $X_k\rightarrow +\infty$ yields an infimum. 
For example, pick one truth table where the convexified $\phi-$risk has an infimum, the truth table of $G_2, G_3$ leads an infimum too: indeed

$
\begin{array}{cccccc}
  & \alpha_0 & \alpha_7 & \alpha_3 & \alpha_4 & \alpha_5\\
G_1   \hspace{-.15cm} & -1  & +1  & -1  & +1   & +1\\
G_2  \hspace{-.15cm}  & -1  & +1  & +1  & -1   & -1 \\
G_3  \hspace{-.15cm}  & -1  & +1  & +1  & -1   & +1 \\
\end{array} \rightarrow
$
$
\begin{array}{ccccc}
  & \alpha_0+\alpha_4 & \alpha_7 + \alpha_3 & \alpha_5 & 0\\
G_2  \hspace{-.15cm}  & -1  & +1   & -1 & +1\\
G_3  \hspace{-.15cm}  & -1  & +1  & +1  & -1\\
\end{array}
$

and we apply lemma \ref{lemma:eu}.
\item if one pair $i$ is full, the three other pairs contain, each, only one element. We have cases of minimum and cases of infimum
\begin{enumerate}
\item Consider the following truth table, where there is a unique point of minimum
$
\begin{array}{cccccc}
  & \alpha_0 & \alpha_7 & \alpha_4 & \alpha_2 & \alpha_1\\
G_1   \hspace{-.15cm} & -1  & +1  & +1  & -1   & -1\\
G_2  \hspace{-.15cm}  & -1  & +1  & -1  & +1   & -1 \\
G_3  \hspace{-.15cm}  & -1  & +1  & -1  & -1   & +1 \\
\end{array}
$.
The three couples $(G_1, G_2)$, $(G_1, G_3)$, $(G_2, G_3)$ each yield an unique point of minimum for the convexified cost function.

\item Consider the truth table, where the convexified cost function has an infimum:
$
\begin{array}{cccccc}
  & \alpha_0 & \alpha_7 & \alpha_4 & \alpha_2 & \alpha_6\\
G_1   \hspace{-.15cm} & -1  & +1  & +1  & -1   & +1\\
G_2  \hspace{-.15cm}  & -1  & +1  & -1  & +1   & +1 \\
G_3  \hspace{-.15cm}  & -1  & +1  & -1  & -1   & -1 \\
\end{array}
$.
In this case both couples $(G_1, G_3)$ and $(G_2, G_3)$ yield an infimum.
\end{enumerate}
\end{itemize}
\item four entries non zero
\begin{itemize}
\item two pairs full: existence of multiple points of minima
\item In the case of one pair full and two non zero entries in two of the three others: the convexified cost function has an infimum. Consider the following example of truth table:
$
\begin{array}{ccccc}
  & \alpha_0 & \alpha_7 & \alpha_4 & \alpha_2\\
G_1   \hspace{-.15cm} & -1  & +1  & +1  & -1\\
G_2  \hspace{-.15cm}  & -1  & +1  & -1  & +1\\
G_3  \hspace{-.15cm}  & -1  & +1  & -1  & -1 \\
\end{array}
$.

The two couples $(G_1, G_3)$ and $(G_2, G_3)$ yield an infimum.
\item In the case where each pair contains a non zero element, we can have either an infimum or a minimum.
\begin{enumerate}
\item Truth table leading to a minimum:
$
\begin{array}{ccccc}
  & \alpha_7 & \alpha_4 & \alpha_2&\alpha_1\\
G_1   \hspace{-.15cm} & +1  & +1  & -1 &-1\\
G_2  \hspace{-.15cm}   & +1  & -1  & +1&-1\\
G_3  \hspace{-.15cm}  & +1  & -1  & -1&+1 \\
\end{array}
$.

The three couples $(G_1, G_2)$, $(G_1, G_3)$, $(G_2, G_3)$ each yield an unique point of minimum for the convexified cost function.

\item Truth table leading to an infimum:
$
\begin{array}{ccccc}
  & \alpha_7 & \alpha_4 & \alpha_2&\alpha_6\\
G_1   \hspace{-.15cm} & +1  & +1  & -1 &+1\\
G_2  \hspace{-.15cm}   & +1  & -1  & +1&+1\\
G_3  \hspace{-.15cm}  & +1  & -1  & -1&-1 \\
\end{array}
$.

The three couples $(G_1, G_2)$, $(G_1, G_3)$, $(G_2, G_3)$ each yield an infimum for the convexified cost function.
\end{enumerate}
\end{itemize}
\item Three entries are non zero: always a point of infimum.
\item Two entries are non zero: if it is in a pair, this yields a non unique point of minimum, if it is not in a pair, it yields an infimum equal to zero.
\item One entry is non zero: it yields an infimum equal to zero.

In the last three cases, at least one column is empty for all pairs of classifiers.

This ends the proof of the second item of proposition \ref{prop:equivalence-inf}.
\begin{flushright}
  $\square$ 
  \end{flushright}
\end{enumerate}

\end{enumerate}

\section{Proof of theorem \ref{equations-minimum-frontiers:lb} and of proposition \ref{equations-frontiers:lb}}
\label{proof:equations-minimum-frontiers:lb}
\begin{itemize}
  \item In the case of the boost loss function, the system of Euler equations that one has to solve is
$$\bs\ba{l}-\alpha_0e^{-X_0}+\alpha_7e^{X_0}-\alpha_1e^{-X_1}+\alpha_6e^{X_1}=0\\
-\alpha_0e^{-X_0}+\alpha_7e^{X_0}-\alpha_2e^{-X_2}+\alpha_5e^{X_2}=0\\
\alpha_0e^{-X_0}-\alpha_7e^{X_0}-\alpha_3e^{-X_3}+\alpha_4e^{X_3}=0\\
e^{X_0+X_3}=e^{X_1+X_2}.\ea\es$$
Denote by $x= \alpha_0e^{-X_0}-\alpha_7e^{X_0}$. This system writes

$\bs\ba{l}-\alpha_1e^{-X_1}+\alpha_6e^{X_1}=x\\
-\alpha_2e^{-X_2}+\alpha_5e^{X_2}=x\\
-\alpha_3e^{-X_3}+\alpha_4e^{X_3}=-x\\
\alpha_0e^{-X_0}-\alpha_7e^{X_0}=x\\
e^{X_0+X_3}=e^{X_1+X_2}\ea\es \leftrightarrow 
\bs\ba{l}\alpha_6e^{X_1}=\frac{x}{2}+\sqrt{(\frac{x}{2})^2+\alpha_1\alpha_6}\\
\alpha_5e^{X_2}=\frac{x}{2}+\sqrt{(\frac{x}{2})^2+\alpha_2\alpha_5}\\
\alpha_4e^{X_3}=-\frac{x}{2}+\sqrt{(\frac{x}{2})^2+\alpha_3\alpha_4}\\
\alpha_7e^{X_0}=-\frac{x}{2}+\sqrt{(\frac{x}{2})^2+\alpha_0\alpha_7}\\
e^{X_0+X_3}=e^{X_1+X_2}.\ea\es$

Plugging the four first relations in the last one, one deduces
$(\frac{x}{2}+\sqrt{(\frac{x}{2})^2+\alpha_1\alpha_6})(\frac{x}{2}+\sqrt{(\frac{x}{2})^2+\alpha_2\alpha_2})\alpha_4\alpha_7=(-\frac{x}{2}+\sqrt{(\frac{x}{2})^2+\alpha_3\alpha_4})(-\frac{x}{2}+\sqrt{(\frac{x}{2})^2+\alpha_0\alpha_7})\alpha_6\alpha_5,$

that is, multiplying by the conjugate expressions on the right hand side
$$\prod_{j=0}^4(\frac{x}{2}+\sqrt{\left(\frac{x}{2}\right)^2+\alpha_j\alpha_{7-j}})= \alpha_0\alpha_3\alpha_5\alpha_6.$$
The equation of the frontier $X_0=0$ yields $a=x$, hence the required relation. The equation of the frontier $X_1=0$ yields $b=-x$, the equation of the frontier $X_2=0$ yields $c=-x$, the equation of the frontier $X_3=0$ yields $d=x$, hence the results.

\item In the case of logit, consider $a_0=\ln2(\alpha_0\phi'(X_0)-\alpha_7\phi'(-X_0))= \frac{\alpha_7e^{X_0}-\alpha_0}{1+e^{X_0}}$, that is $e^{X_0}=\frac{\alpha_0+a_0}{\alpha_7-a_0}$. The equations for the point of minimum $(X_1, X_2, X_3)$ are

$$
\bs\begin{array}{l} \frac{\alpha_6e^{X_1}-\alpha_1}{1+e^{X_1}}=-a_0\\
\frac{\alpha_5e^{X_2}-\alpha_2}{1+e^{X_2}}=-a_0\\
\frac{\alpha_4e^{X_3}-\alpha_3}{1+e^{X_3}}=a_0\\
e^{X_1}e^{X_2}= e^{X_0}e^{X_3}.\end{array}\es
\Leftrightarrow
\bs\begin{array}{l}e^{X_1}=\frac{\alpha_1-a_0}{\alpha_6+a_0}\\
e^{X_2}=\frac{\alpha_2-a_0}{\alpha_5+a_0}\\
e^{X_3}=\frac{\alpha_3+a_0}{\alpha_4-a_0}\\
e^{X_1}e^{X_2}= e^{X_0}e^{X_3}\end{array}\es
\Leftrightarrow
\bs
\begin{array}{l}e^{X_1}=\frac{\alpha_1-a_0}{\alpha_6+a_0}\\
e^{X_2}=\frac{\alpha_2-a_0}{\alpha_5+a_0}\\
e^{X_3}=\frac{\alpha_3+a_0}{\alpha_4-a_0}\\
\frac{\alpha_1-a_0}{\alpha_6+a_0}\frac{\alpha_2-a_0}{\alpha_5+a_0}\frac{\alpha_4-a_0}{\alpha_3+a_0}= \frac{\alpha_0+a_0}{\alpha_7-a_0}.\end{array}\es$$
The equation that one has to solve is thus
\begin{equation}
-(\alpha_1-a_0)(\alpha_2-a_0)(\alpha_4-a_0)(\alpha_7-a_0)+(\alpha_0+a_0)(\alpha_3+a_0)(\alpha_5+a_0)(\alpha_6+a_0)=0.
\end{equation}

Using $\sum_j\alpha_j=1$, one checks that this equation writes
\begin{equation}
  a_0^3+A a_0^2+ B a_0+ C=0.
\end{equation}
Replacing $a_0$ by $-a/2$ for $X_0=0$ or by $b/2$ for $X_1=0$ or by $c/2$ for $X_2=0$ or by $-d/2$ for $X_3=0$, one obtains the equations of the frontiers $X_j=0$.
\end{itemize}

\end{document}